\documentclass{article}

    \PassOptionsToPackage{numbers, compress}{natbib}

\usepackage[preprint]{neurips_2024}

\usepackage[utf8]{inputenc} %
\usepackage[T1]{fontenc}    %
\usepackage{hyperref}       %
\usepackage{url}            %
\usepackage{booktabs}       %
\usepackage{amsfonts}       %
\usepackage{nicefrac}       %
\usepackage{microtype}      %
\usepackage{xcolor}         %
\usepackage{subfigure}

\usepackage{amsmath}
\usepackage{amssymb}
\usepackage{mathtools}
\usepackage{amsthm}

\usepackage{hyperref}
\usepackage{url}
\usepackage{tabularx}

\usepackage[capitalize,noabbrev]{cleveref}

\theoremstyle{plain}
\newtheorem{theorem}{Theorem}[section]

\newtheorem{lemma}[theorem]{Lemma}

\theoremstyle{definition}
\newtheorem{definition}[theorem]{Definition}

\theoremstyle{remark}
\newtheorem{remark}[theorem]{Remark}

\usepackage{algorithm}
\usepackage{algorithmic}

\title{\textit{What's the Magic Word?} A Control Theory of LLM Prompting}

\author{%
  Aman Bhargava\thanks{Use footnote for providing further information
    about author (webpage, alternative address)---\emph{not} for acknowledging
    funding agencies.} \\
  California Institute of Technology\\
  Pasadena, CA, USA \\
  \texttt{abhargav[at]caltech[dot]edu} \\
  \And
  Cameron Witkowski \\
  University of Toronto \\
  Toronto, ON, Canada \\
  \texttt{cameron.witkowski[at]mail.utoronto.ca} \\
  \And
  Shi-Zhuo Looi \\
  California Institute of Technology \\
  Pasadena, CA, USA \\
  \texttt{looi[at]caltech[dot]edu} \\
  \And
  Matt Thomson \\
  California Institute of Technology \\
  Pasadena, CA, USA \\
  \texttt{mthomson[at]caltech[dot]edu} \\
}

\begin{document}

\maketitle

\begin{abstract}
Prompt engineering is crucial for deploying LLMs but is poorly understood mathematically. 
We formalize LLM systems as a class of discrete stochastic dynamical systems to explore prompt engineering through the lens of control theory. 
We offer a mathematical analysis of the limitations on the controllability of self-attention as a function of the singular values of the parameter matrices.
We present complementary empirical results on the controllability of a panel of LLMs, including Falcon-7b, Llama-7b, and Falcon-40b. 
Given initial state $\mathbf x_0$ from Wikitext and prompts of length $k \leq 10$ tokens, we find that the ``correct'' next token is reachable at least 97\% of the time, and that the top 75 most likely next tokens are reachable at least 85\% of the time. 
Intriguingly, short prompt sequences can dramatically alter the likelihood of specific outputs, even making the least likely tokens become the most likely ones. 
This control-theoretic analysis of LLMs demonstrates the significant and poorly understood role of input sequences in steering output probabilities, offering a foundational perspective for enhancing language model system capabilities.
\end{abstract}

\section{Introduction}

LLMs pre-trained on unsupervised next token prediction objectives exhibit unprecedented dynamic reprogrammability achieved through ``prompting'', often referred to as zero-shot learning \citep{gpt3, wei2022emergent, hagendorff2023machine, noever2023numeracy, openai2023gpt4, GPT3_5_blog}. 
These capabilities appear to emerge as the model's size, training data, and training time are scaled.
The dynamic reprogrammability of LLMs is akin to the adaptable computational capacities observed in biological systems. 
This feature finds applications across domains such as machine translation \citep{wang2023document}, code generation \citep{code_llama}, and chatbots \cite{rlhf_chatbot}. 
A rigorous understanding of the prompt's influence over LLM generation would be of great utility for understanding LLMs and building more robust and capable systems leveraging LLMs.

Strategies for controlling pre-trained LLM generation today fall into three broad categories \cite{survey_controllable_text_gen}: 
\begin{enumerate}
    \item \textbf{Input Optimization (Prompting): } Adjusting the input tokens (e.g., rewording the prompt) to improve subsequent text generation. 
    \item \textbf{Model Optimization: } Adjusting the weights of the network (e.g., fine-tuning, RLHF) to improve model behavior during inference. 
    \item \textbf{Post-processing: } Adjusting or re-ranking generated text (e.g., surrogate ranking algorithm). 
\end{enumerate}

Of all these approaches, input optimization (i.e., prompting) is the least invasive and lowest-cost method -- and the least understood.
Prompt optimization is also deeply connected to the zero-shot capabilities of LLMs -- the mysterious emergent capabilities of LLMs such as problem-solving, knowledge retrieval, reasoning, and apparent general intelligence \citep{bubeck2023sparks}. With such a view, we seek to characterize the controllability of LLMs via prompting (Figure~\ref{fig:big_picture}).

\begin{figure}[htbp]
    \centering
    \includegraphics[width=\columnwidth]{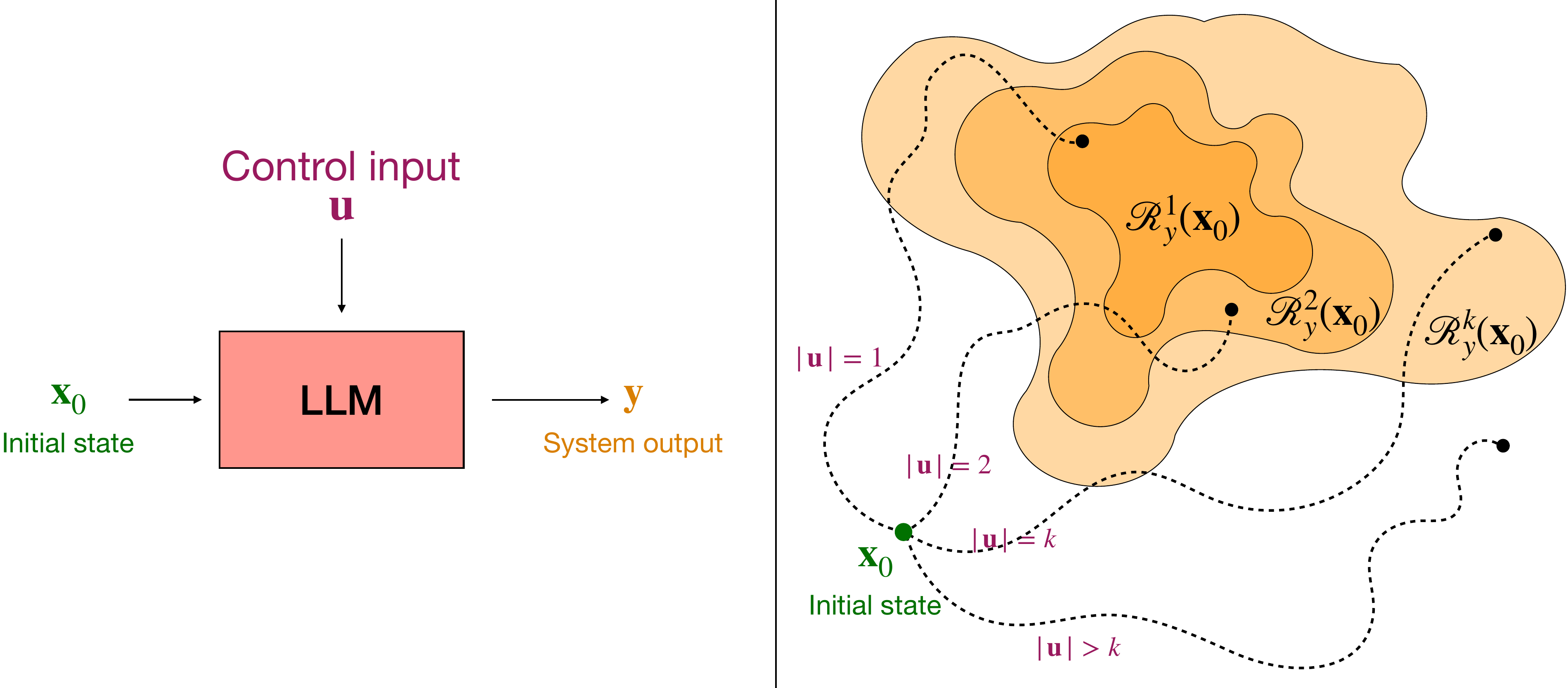}
    \caption{Illustration of the control-theoretic approach to LLM prompt engineering. 
    \textbf{Left:} the LLM system diagram mapping an initial state $\mathbf{x}_0$ to a system output $\mathbf{y}$ under the influence of a control input $\mathbf{u}$ (all token sequences). \textbf{Right}: sketch of the reachable output sets $R_y^k(\mathbf{x}_0)$ for varying control input lengths $k$. 
    }

    \label{fig:big_picture}
\end{figure}

\subsection{Contribution}

We formalize LLM systems in the mathematical framework of control theory in Section~\ref{sec:ctrl_theory_llms}.
Our analysis focuses on the reachable set of outputs $\mathcal R_y (\mathbf x_0)$ for an LLM system.
The reachable set is a fundamental concept in control theory that underlies notions of controllability, stability, and observability (cf. Appendix~\ref{sec:definitions_ctrl}). 
The reachable output set $R_y(\mathbf x_0)$ is the set of output sequences $\mathbf y$ for which there exists a control input sequence $\mathbf u^*$ that steers the LLM from initial state $\mathbf x_0$ to output $\mathbf y$ (cf. Definitions~\ref{def:llm-reachable-output-set},~\ref{def:reachable-output-set}). 

Our mathematical results in Section~\ref{sec:theorem} prove an upper bound on the contents of the reachable output set for a self-attention head as a function of the singular values of its parameter matrices. 
Since self-attention is the only component in a transformer block where significant information is exchanged between token representations, this bound provides a foothold for analysis of LLM controllability from the perspective of mechanistic interpretability (e.g., \cite{bricken2023monosemanticity, chefer2021transformer, conmy2023automated}). 
Our result represents an analytically computable necessary condition for an output to be in the reachable set (Equation~\ref{eqn:attention-reachability-condition}).

Our empirical results apply state-of-the-art prompt optimization techniques (Section~\ref{sec:methods}) to demonstrate a lower bound on the contents of the reachable output set for a panel of LLMs, including Llama-7b \citep{llama_1}, Falcon-7b, and Falcon-40b \citep{falcon40b}.
Specifically, we sample initial states $\mathbf x_0$ from the Wikitext dataset \citep{wikitext} and probe the reachable output tokens $y$ under length-constrained control input sequences $\mathbf u: |\mathbf u| \leq k$. 
The length constraint $k$ is highly relevant for \textit{optimal control} of LLMs, as prompts with fewer tokens require fewer computation and memory resources. 
We find that the reachable output set contains the ``correct'' next Wikitext token following $\mathbf x_0$ over 97\% of the time with prompts of $k\leq 10$ tokens. We expand our analysis of the contents of $R_y(\mathbf x_0)$ by sampling target output tokens $y$ based on the LLMs initial estimate of output likelihood $P_{LM}(y | \mathbf x_0)$. 
We find that the top 75 most likely output tokens $y$ are reachable at least 85\% of the time with prompts of $k\leq 10$ tokens. 
Intriguingly, some tokens drawn from the set of \textit{least} likely outputs are controllable to the most likely output with $k\leq 4$ control input tokens. 
Our results suggest that prior likelihood-based metrics, such as cross-entropy loss, cannot guarantee exclusion from the reachable set, emphasizing the gap in our current understanding of LLM systems and control. Implications of our results and open questions in LLM control theory are further discussed in Section~\ref{sec:discussion}.

\section{Related Work}
Much of the work on prompt optimization is concerned with finding prompts that induce higher LLM performance on ``fill-in-the-blank'' or ``cloze'' tasks \citep{cloze_1959}. 
One can frame a range of tasks including knowledge retrieval \citep{LAMA_dataset}, reasoning \citep{babl_dataset_reasoning}, and sentiment analysis \citep{wang2023chatgpt} as fill-in-the-blank tasks: 
\begin{itemize}
    \item \textbf{Knowledge Retrieval: } \textit{``The Titanic sank in the year \textbf{[MASK]}.''} (Answer: ``1912'')
    \item \textbf{Reasoning: } \textit{``A is taller than B. B is taller than C. Is A taller than C? \textbf{Answer: [MASK]}''} (Answer: ``Yes'')
    \item \textbf{Sentiment Analysis: } \textit{``I am sad today. \textbf{The sentiment of the previous sentence was [MASK]}''} (Answer: ``Negative'')
\end{itemize}

Notably, there is some freedom in the bolded ``prompt text'' that surrounds the question to convert it into a ``fill-in-the-blank'' task. As it turns out, the prompt tokens have a large effect on LLM performance \citep{gpt3, survey_controllable_text_gen, jiang2020know}.

Modern prompt optimization algorithms generally consist of two iterated steps: a sampling step where new prompts are generated and a testing step where the utility of the new prompts is evaluated, and the best are selected for the next iteration. 
Algorithms primarily differ in the sampling procedure, where various heuristics may be used to pick high-value swaps \citep{wen2023hard, zhou2023large, reynolds2021prompt}.  
Overall, AutoPrompt and its derivative algorithms have been the most numerically successful prompt optimization methods, with the greedy coordinate gradient (GCG) algorithm having state-of-the-art performance \citep{zou2023universal}.

\paragraph{The AutoPrompt Family: } AutoPrompt \citep{shin2020autoprompt} pioneered the current wave of prompt optimization. 
Shin \textit{et al} propose a prompt optimization technique and demonstrate its effectiveness for engineering prompts to improve LLM performance on knowledge and sentiment analysis tasks. 
At its core, the AutoPrompt algorithm leverages gradient information at the token embedding layer to inform iterative token exchanges within the prompt. 
This method was extended in \cite{zou2023universal} as the greedy coordinate gradient (GCG) algorithm. Taking inspiration from adversarial examples \citep{goodfellow2015explaining}, Zou \textit{et al} applied this AutoPrompt variant to generate ``jailbreak'' prompts that cause aligned LLMs to generate objectionable content. 

\paragraph{Other Prompt Optimization Methods: } Other investigations on LLMs as prompt optimizers \citep{zhou2023large} and further analysis of manual prompt optimization \citep{reynolds2021prompt} are informative but do not exceed the AutoPrompt family's performance. Some other methods include GBDA \citep{guo2021gradientbased}, an approach based on the Gumbel-Softmax reparametrization, the PEZ algorithm \citep{wen2023hard}, which directly optimizes embeddings via gradient information, and FluentPrompt \citep{shi2022human}, which differs from AutoPrompt by incorporating Langevin dynamics. Another family of techniques relating closely to our work is RL-Based prompt optimization methods \citep{deng2022rlprompt, hao2023optimizing, zhang2022tempera, kim2023multiprompter}. Such methods seek to optimize a prompt generation policy to maximize some reward signal, using a host of off the shelf reinforcement learning algorithms. Despite the variety of alternatives, GCG retains state-of-the-art performance.

\paragraph{Control Theory for LLMs: } To our knowledge, the only other work to date on the controllability or reachability of LLM text generation is \cite{soatto2023taming}. 
Soatto et al analyze the controllability of LLMs in terms of ``meaningful sentences'', defined as the sigma-algebra generated by snippets of text written on the Internet. 
Their empirical analysis revolves around demonstrating that LLMs are capable of attributing meaning. 
The theoretical analysis of LLM controllability is limited to ``meaningful sentences'', eliminating the possibility of out-of-distribution inputs and outputs.
These restrictions render their results challenging to leverage toward a practical understanding of LLM controllability. 
As stated in Section 5.5 of \cite{soatto2023taming}, ``If fed gibberish, the well-trained bot operates out of distribution, which does not allow predicting the reachable set''. 
We situate our work as a practically oriented exploration of LLM controllability. 
Motivated by challenges in developing LLM systems, we do not eliminate ``meaningless sentences'' from the state space or input space. 
We aim to establish a rigorous, general framework for understanding LLM systems and controllability that is amenable to the development of theory and practical engineering insights on systems design.

\section{Control Theory for LLMs}
\label{sec:ctrl_theory_llms}
Control theory originates from the study of automatic control systems in engineering. It seeks to understand how a ``plant'' system can be influenced toward a desired state using a ``control signal'' -- often in the presence of disturbances and uncertainty. 

Control theory is central to a variety of engineering problems, from electrical engineering to autopilot to telecommunications to manufacturing. Surprisingly, control theory has also been highly applicable to a diverse range of scientific disciplines. Analyzing systems through the lens of controllability has proven fruitful for generating insight into biological systems such as cell signaling pathways and neural networks \citep{yi2000robust}, the economics of central banking \citep{anicta2011introduction}, and controlling the spread of infectious diseases \citep{virus_control}. One of the central benefits of studying systems via controllability is that a range of questions and problems naturally emerge from the framing: \textit{when is control possible? What is the cost of control? How computationally intensive is control?} These questions are both practically useful and often lead to fundamental insights about the nature of the system in question.

To develop a control theory of LLMs, we begin with fundamental definitions of systems and control in Appendix~\ref{sec:definitions_ctrl}. 
We extend these fundamentals to define LLM systems (Definition~\ref{def:llm-system}) and outline specific canonical control concepts and problems such as controllability and reachability (Definition~\ref{def:llm-reachable-output-set},~\ref{def:llm-output-controllability}) that arise naturally for LLM systems. 

\paragraph{Language Model Notation: } We denote a causal language model using $P_{LM}$. $P_{LM}$ maps from an ordered list of tokens from a vocabulary set $\mathcal V$ (e.g., $\mathbf x \in \mathcal V^n$) to the probability distribution over the next token $P_{LM}(x_{n+1} |\mathbf x) \in [0,1]^{|\mathcal V|}$. We use $\mathcal V^*$ to denote the set of all possible sequences of any length composed of tokens from $\mathcal V$. The addition operator indicates the concatenation of token sequences. Bolded lowercase variables (e.g., $\mathbf x = [x^1, \dots, x^n]$) denote token sequences while unbolded lowercase variables refer to individual tokens (e.g., $x\in \mathcal V$). The length of a token sequence is denoted $|\mathbf x|$.

While LLMs are at times leveraged in a manner that masks the iterative aspects of generation, the reality is that token generation and externally imposed ``control input'' sequences are generated and processed sequentially, leading to non-trivial system dynamics. 
Several key differences remain between LLM-based systems and systems typically modeled through ordinary differential equations (ODEs), which have long been a cornerstone in the study of continuous-time dynamical systems:
\begin{enumerate}
    \item \textbf{Discrete state and time: } LLM systems operate on sequences of discrete tokens over a discrete time set, in contrast to the continuous state spaces and time sets studied in classical control theory.
    \item \textbf{Shift-and-Grow State Dynamics: } Whereas the system state in an ODE-based system has a fixed size over time, the system state $\mathbf x(t)$ for LLM systems grows as tokens are added to the state sequence. 
    \item \textbf{Mutual exclusion on control input token vs. generated token: } The LLM system state $\mathbf x(t)$ is written to one token at a time. 
    The newest token is either drawn from the control input $u(t)$ or is generated by the LLM by sampling $x'\sim P_{LM}(x' | \mathbf x(t))$. 
    This differs from traditional discrete stochastic systems, where the control sequence and internal dynamics generally affect the state synchronously. 
\end{enumerate}

We begin by rigorously defining LLM systems with user input, drawing from the abstract mathematical definition of a system (Definition~\ref{def:system}).

\begin{definition}[LLM System with Control Input]
\label{def:llm-system} 
An autoregressive LLM system with control input $\Sigma = (\mathcal V, P_{LM})$ consists of:
\begin{itemize}
    \item $\mathcal T = \mathbb N$ -- The \textbf{time set} is the natural numbers. 
    \item $\mathcal X = \mathcal V^*$ -- The \textbf{state space} consists of all possible token sequences of any length drawn from $\mathcal V$. We denote the state at time $t$ as $\mathbf x(t) = [x^0(t), \dots, x^t(t)]$.
    \item $\mathcal U = \mathcal V \cup \varnothing$ -- The \textbf{input} takes values from the vocabulary set $\mathcal V$ or null. 
    \item $\phi: \mathcal X \times \mathcal U \times \mathcal T^2 \to \mathcal X$ -- The \textbf{transition map} is 
    \begin{align}
        \label{eqn:llm-transition-map}
        &\phi(\mathbf x(t), u(t), t, t+1) = \begin{cases}
            \mathbf x(t) + u(t) & \text{ if } u(t) \neq \varnothing \\ 
            \mathbf x(t) + x'  & \text{ else } 
        \end{cases}
    \end{align}
    \textit{where $x' \sim P_{LM}(x' | \mathbf x(t))$}.
    Note that the general multi-step transition map $\phi(\mathbf x(t), u, t, t+N)$ can be achieved by iterating equation~\ref{eqn:llm-transition-map} for control sequences $\mathbf u$ defined over the interval $[t, t+N]$. 

    \item $h(\mathbf x(t); r) = [x^{t-r}(t), \dots, x^t(t)]$ -- The \textbf{readout map} returns the most recent $r$ tokens from state $\mathbf x(t)$. 
\end{itemize}
\end{definition}

We note that this LLM system definition is generalizable to a variety of LLM augmentations, including chain-of-thought \citep{wei2023chainofthought}, retrieval-augmented generation \citep{lewis2020retrieval}, and chatbot interaction. For example, chain-of-thought is equivalent to sampling the readout map $h(x(t), r)$ at time $T > |\mathbf u| + |\mathbf x_0| + r$ for prompt $\mathbf u$ and initial state $\mathbf x_0$. A similar formulation may be applied to LLM systems endowed with programmatic tools (e.g., \cite{patil2023gorilla}). 

In Definition~\ref{def:llm-system}, we assume that the control input gets to ``decide'' whether to yield token generation to the LLM ($u(t) = \varnothing$) or override the LLM and add some token $u(t) \neq \varnothing$ to the state $\mathbf x(t)$. 
This assumption generally holds when building LLM systems, though it may not hold when using existing systems (e.g., via non-streaming API). 
When discussing finite-length control inputs -- e.g., the family of $k$-long input sequences $\mathbf u \in \mathcal V^k$ -- the value of $u(\ell) : \ell > k$ is implicitly $\varnothing$ unless otherwise stated.  

While next token generation $x' \sim P_{LM} (x' | \mathbf x(t))$ in equation~\ref{eqn:llm-transition-map} is probabilistic, we may render the system deterministic by sampling with zero temperature (i.e., greedy decoding). The greedy decoding assumption provides a foothold to analyze the reachable sets and controllability of LLM systems without invoking notions of stochastic control as in \cite{sivaramakrishnan2023stochastic, soatto2023taming}. 
Moreover, it remains connected to temperature-based stochastic decoding strategies as a limiting case of temperature-based sampling as zero-temperature sampling.

We now extend Definition~\ref{def:output-reachability} to define output controllability for LLM systems: 

\begin{definition}[LLM Output Reachability]
\label{def:llm-output-reachability}
Output token sequence $\mathbf y \in \mathcal V^r$ is reachable from initial state $\mathbf x_0\in \mathcal V^*$ for LLM system $\Sigma(\mathcal V, P_{LM})$ iff there exists some time $T$ and input $\mathbf u^* \in \mathcal U^k$ for some $k+|\mathbf x_0| \leq T$ that steers the LLM from initial state $\mathbf x_0$ to output $\mathbf y = h(\mathbf x(T), r)$ at time $T$.
\end{definition}

We disregard the trivial solution wherein the control input $\mathbf u^*(t)$ overrides the LLM to force the state sequence to take on the desired output value $\mathbf y$.
We focus on the case of immediate generation, where $T=k+|\mathbf x_0| + r$.

The reachable output set definition for LLM systems follows from Definition~\ref{def:reachable-output-set}: 

\begin{definition}[LLM Reachable Output Set]
\label{def:llm-reachable-output-set}
The reachable output set from initial state $\mathbf x_0\in \mathcal V^*$ for LLM system $\Sigma = (\mathcal V, P_{LM})$ is denoted $R_y(\mathbf x_0)$ and consists of all reachable outputs $\mathbf y \in \mathcal V^*$ from initial state $\mathbf x_0$. 
\end{definition}

Output controllability for LLMs follows from Definition~\ref{def:output-controllability}:

\begin{definition}[LLM Output Controllability]
\label{def:llm-output-controllability}
An LLM system $\Sigma = (\mathcal V, P_{LM})$ is output controllable iff, for every initial state $\mathbf x_0 \in \mathcal V^*$, the reachable output set $\mathcal R_y(\mathbf x_0) = \mathcal V^*$.
\end{definition}

The turn-based nature of writing to the LLM state sequence $\mathbf x(t)$ invites the question of whether the prompt $\mathbf u$ should preempt the imposed state $\mathbf x_0$ or come after the state
\footnote{Both situations are reasonable in developing LLM systems: $\mathbf u$ preceding $\mathbf x_0$ may arise when prompting an LLM to complete a partial string $\mathbf x_0$. $\mathbf u$ proceeding $\mathbf x_0$ may arise when prompting an LLM in the presence of an imposed system prompt $\mathbf x_0$. Therefore, how an initial state $\mathbf x_0$ is interleaved with control input $\mathbf u$ is largely a design decision.}. 
We focus our efforts on cases where $\mathbf u$ comes before imposed state sequence $\mathbf x_0$ due to its importance for developing system prompts and controlling text completion-based generation where the desired output is $\mathbf x_0 + \mathbf y^*$ for some desired continuation $\mathbf y^*$ of partial string $\mathbf x_0$.
Due to the costly nature of long prompts, we are especially interested in the existence of prompts $\mathbf u^*$ with minimal length $|\mathbf u^*|$. 

Definitions~\ref{def:llm-reachable-output-set} and ~\ref{def:llm-output-controllability} form the basis for our control theory of LLMs. 
While amenable to theoretical analysis as in Section~\ref{sec:theorem} and \cite{soatto2023taming}, empirical analysis of the reachable set and controllability is challenging due to the intractable size of $\mathcal V^*$.
We propose the following statistical measure of controllability for practically assessing the controllability of an LLM system w.r.t. a dataset $\mathcal D$ under prompt length constraint $|\mathbf u| \leq k$: 

\begin{definition}[$k$-$\epsilon$ Controllability]
\label{def:k-eps}
Consider a dataset of state-output pairs $\mathcal D = \{(\mathbf x_0^i, \mathbf y^i)\}_{i\in [N]}$. An LLM $\Sigma = (\mathcal V, P_{LM})$ is $k$-$\epsilon$ controllable w.r.t. $\mathcal D$ if
\begin{equation}
    \label{eqn:k-epsilon}
    \Pr\{\mathbf y \notin \mathcal R^k_y(\mathbf x_0)\} \leq \epsilon
\end{equation}
For $(\mathbf x_0, \mathbf y) \sim \mathcal D$, where $\mathcal R^k_y(\mathbf x_0^i)$ is the reachable set of outputs as in Definition~\ref{def:llm-reachable-output-set} under the constraint that prompts $\mathbf u$ must have length $|\mathbf u|\leq k$. 
\end{definition}

Our empirical work in Section~\ref{sec:results} explores $k$-$\epsilon$ controllability w.r.t. initial states $\mathbf x_0$ sampled from the Wikitext dataset. 
While empirical analysis of LLM controllability is challenging due to the lack of apparent structure in LLM dynamics and the combinatorially large state space, we may still experimentally establish the \textit{existence} of optimal prompts $\mathbf u^*$ that elicit a given output, and thus establish a lower bound on the content of the reachable set. 
Meanwhile, our theoretical work in Section~\ref{sec:theorem} establishes \text{upper bounds} on the content of the reachable set for self-attention. 
We hope these complementary approaches aid in unifying our understanding of LLM systems.

\section{The Self-Attention Control Theorem}
\label{sec:theorem}

Self-attention is a central component in modern transformer-based language models \citep{gpt3, llama_1, radford2019language, min2023recent}. Introduced in \cite{bahdanau2016neural} and popularized by \cite{vaswani2017attention}, self-attention is the primary component in transformers where token representations exchange information. 
Self-attention mechanisms have significantly advanced the field of natural language processing, enabling models to capture long-range dependencies and achieve impressive performance on various tasks.
Despite the widespread adoption and success of self-attention, the extent to which the outputs of self-attention layers can be precisely controlled via the input sequence remains an open question.

In this section, we present the Self-Attention Control Theorem, which proves bounds for understanding the reachability of self-attention outputs given limited control over the input token representations.

\subsection{Preliminaries}

\begin{definition}[Self-Attention]
\label{def:self-attention}

Self-attention $\Xi$ is parameterized by weight matrices $\boldsymbol{\theta} = (\mathbf W_q, \mathbf W_{\rm key}, \mathbf W_v)$. $\Xi$ is a mapping from $\mathbb{R}^{N \times d_{in}}$ to $\mathbb{R}^{N \times d_{out}}$, where $N$ is the number of input token representations, each of dimensionality $d_{in}$, and $d_{out}$ is the dimensionality of the output token representations.

\begin{equation}
\label{eqn:self-attention}
\Xi(\mathbf{X}; \boldsymbol{\theta}) = \mathbf{D}^{-1} \exp \left( \frac{\mathbf{QK^{\top}}}{\sqrt{d_{\rm key}}} \right) \mathbf{V}
\end{equation}

where $\exp()$ denotes element-wise exponentiation of the matrix entries, $\mathbf W_q, \mathbf W_{\rm key} \in \mathbb{R}^{d_{in} \times d_{\rm key}}$, $\mathbf W_v \in \mathbb{R}^{d_{in} \times d_{out}}$, $\mathbf Q =\mathbf X\mathbf W_q$, $\mathbf K = \mathbf  X \mathbf W_{\rm key}$, $\mathbf V = \mathbf 
 X\mathbf  W_v$, and $\mathbf{D}$ is a diagonal positive definite matrix defined as

\begin{equation}
\label{eqn:denominator-definition}
\mathbf{D} := \text{diag} \left( \exp \left( \frac{\mathbf{QK^{\top}}}{\sqrt{d_{\rm key}}} \right) \mathbf{1}_{N \times 1} \right)
\end{equation}

where $\mathbf{1}_{N \times 1}$ is an $N \times 1$ matrix of ones.
\end{definition}

The parameters and operation of $\Xi$ are independent of the number of token representations $N$. Self-attention is typically applied to discrete token sequences by embedding each token in the sequence as a vector in $\mathbb{R}^{d_{in}}$ to construct the matrix of $N$ token representations $\mathbf{X} \in \mathbb{R}^{N \times d_{in}}$.

We focus on the reachability of output token representations $\Xi(\mathbf{X}; \boldsymbol{\theta})$, where we partition the input $\mathbf{X} \in \mathbb{R}^{(k + m) \times d_{in}}$ into a $k \times d_{in}$ block of control input representations $\mathbf{U}$ and an $m \times d_{in}$ block of imposed state representations $\mathbf{X_0}$ (cf. Definition~\ref{def:llm-system}) where $k+m=N$. 
Thus the complete input matrix $\mathbf{X}$ is a concatenation of the control input $\mathbf{U}$ and the imposed state $\mathbf{X}_0$. 

\begin{align}
    \Xi (\mathbf X; \boldsymbol{\theta}) &= \Xi \begin{pmatrix}\begin{bmatrix}
        \mathbf U \\
        \mathbf X_0
    \end{bmatrix}; \boldsymbol{\theta}\end{pmatrix} = \Xi([\mathbf U; \mathbf X_0]; \boldsymbol{\theta}) 
    \\
    &= \begin{bmatrix}
        \mathbf U' \\
        \mathbf Y
    \end{bmatrix} = [\mathbf U'; \mathbf Y]
\end{align}

We also partition the output $\mathbf{X'} = \Xi (\mathbf{X}; \boldsymbol{\theta}) \in \mathbb{R}^{(k+m) \times d_{in}}$ into a corresponding $k \times d_{out}$ matrix $\mathbf{U'}$ and an $m \times d_{out}$ matrix $\mathbf{Y}$. 

We aim to characterize the reachable set of output representations $\mathbf{Y} \in \mathcal{R}_y^k(\mathbf{X_0})$ under $m$ imposed input representations $\mathbf{X_0}$ and $k$ controllable input representations $\mathbf{U}$. 
Although the reachable set is now a set of continuous-valued output representation matrices in $\mathbb{R}^{m \times d_{in}}$, we can readily adapt Definitions~\ref{def:llm-reachable-output-set}-\ref{def:llm-output-reachability} to define the reachable set for these conditions: 

\paragraph{Reachability for Self Attention: } Following from the original output reachability definition (Definition \ref{def:llm-output-reachability}), let ${\mathbf Y}^*\in \mathbb{R}^{m\times d_{out}}$ be the desired output. 
We consider $\mathbf Y^*$ reachable from initial state $\mathbf X_0$ if there exists some $\mathbf U$ that steers the output of $\Xi\big([\mathbf U; \mathbf X_0]; \boldsymbol{\theta}\big)$ to output $[\mathbf U'; \mathbf Y]$ such that $\mathbf Y = \mathbf Y^*$.

\subsection{The theorem and its motivation}

Our approach is to split the output $\mathbf Y$ into two parts, $\mathbf Y=\mathbf Y_u+\mathbf Y_x$, corresponding to the control input and imposed state, respectively. $\mathbf Y_x$ can be bounded as a function of $\mathbf X, k$, and $\boldsymbol{\theta}$. $\mathbf Y_u$ is the remaining component arising from $\mathbf U$. 
Each of the two parts $\mathbf Y_u$ and $\mathbf Y_x$ is split into two further components, one orthogonal to $\mathbf Y^*$ and one parallel to it. For instance, we denote the orthogonal part of $\mathbf Y_x$ by $\mathbf Y_{x, \perp}$. Thus we have 
\begin{align*}
    \mathbf Y&=\mathbf Y_{u}+\mathbf Y_{x}\\
    &=(\mathbf Y_{u,||}+\mathbf Y_{u,\perp}) + (\mathbf Y_{x,||}+\mathbf Y_{x,\perp})
\end{align*}
After rearranging, we have $\mathbf Y=
(\mathbf Y_{u,||}+\mathbf Y_{x,||}) + (\mathbf Y_{u,\perp}+\mathbf Y_{x,\perp})
    \in \operatorname{span}(\mathbf Y^*)\oplus \operatorname{span}(\mathbf Y^*)^\perp$.
If the desired output is reachable, then $\mathbf Y_{u,\perp}+\mathbf Y_{x,\perp}= \mathbf 0$ and also $\|\mathbf Y_{u,\perp}\|=\|\mathbf Y_{x,\perp}\|$ (see Appendix~\ref{app:more_general_theorem}). \\

\begin{figure}[ht]
    \centering
    \begin{minipage}{0.5\textwidth}
        \includegraphics[width=\linewidth]{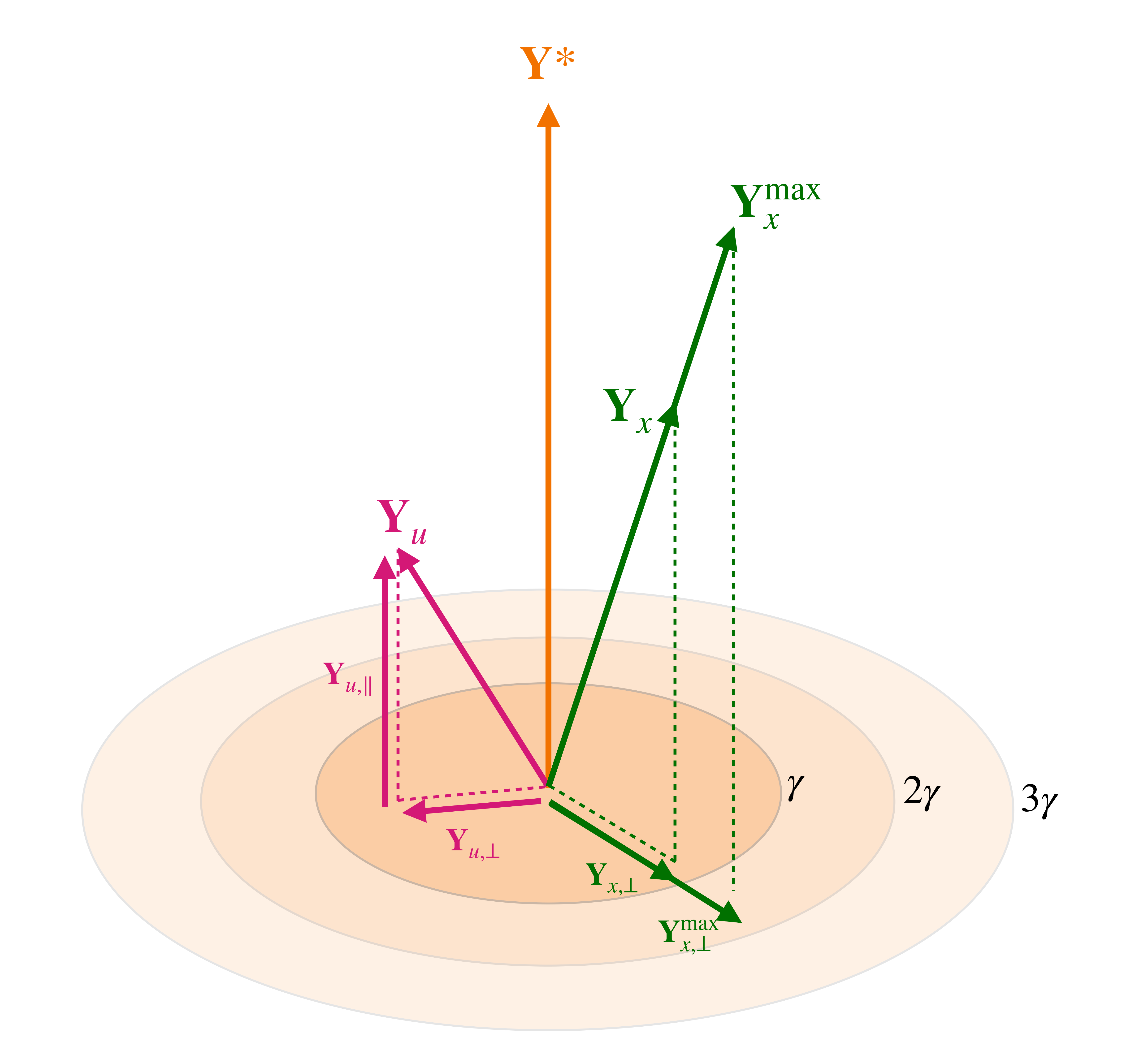}
    \end{minipage}
    \caption{
        \label{fig:attention_figure}
        Visualization of Theorem~\ref{thm:attention-control} $\mathbf{Y}^*$ and components of $\mathbf{Y}_u, \mathbf{Y}_x$. 
        If $\|\mathbf{Y}_{x,\perp}^{\max,i} \|$ exceeds $k\gamma$, then no prompt of length $\leq k$ can steer the self-attention to output $\mathbf{Y}^*$ given imposed $\mathbf{X}_0$ and constraints on $\|\mathbf{U}^i\| \leq M_u$.
    }
\end{figure}

\begin{theorem}[Self-Attention Control Theorem, proved in Appendix~\ref{app:proof-control-llms}]
\label{thm:attention-control}

Consider a self-attention layer with input $\mathbf{X} \in \mathbb{R}^{m \times d}$ and control input $\mathbf{U} \in \mathbb{R}^{k \times d}$, where $m$ is the number of imposed tokens, $k$ is the number of control tokens, and $d$ is the token embedding dimension. Let $\mathbf{Y}^* \in \mathbb{R}^{m \times d}$ be the desired output, and let $\mathbf{Y} \in \mathbb{R}^{m \times d}$ be the actual output of the self-attention layer.

Let $\mathbf Y_x^{\operatorname{max}} = \Xi(\mathbf X_0; \boldsymbol{\theta})$ be the output of the self-attention layer given only the imposed state $\mathbf X_0$. 
As before, we denote the $i$-th row of the orthogonal component of $\mathbf Y_x^{\operatorname{max}}$ to the %
desired $\mathbf Y^{*}$ as $\mathbf Y_{x,\perp}^{\operatorname{max},i}$. 

Then $\mathbf Y^*$ is unreachable for any control input $\mathbf U$ if, for any $i\in \{1,\dots,m\}$, 

\begin{equation}
\label{eqn:attention-reachability-condition}
\|\mathbf{Y}_{x,\perp}^{\operatorname{max},i}\| > k \gamma_i(\mathbf X_0, \boldsymbol{\theta})
\end{equation}
where 
\begin{equation}
\gamma_i(\mathbf X_0, \boldsymbol{\theta}) := \frac{e^\alpha}{g_i} \sigma_v M_u, \quad \alpha = \sigma_q\sigma_{\rm key} M_uM_x/\sqrt{d_\mathrm{key}} 
\end{equation}

\begin{equation}
    g_i(\mathbf X_0,\boldsymbol{\theta}) := \sum_{j=1}^m\exp\left((\mathbf X_0)^{i}\mathbf W_q\mathbf W_\mathrm{key}^\top (\mathbf X_0)^{j\top}/\sqrt{d_\mathrm{key}}\right),
\end{equation}
$\sigma_v,\sigma_q$ and $\sigma_{\rm key}$ being the maximum singular values of the value, query and key projection matrices, respectively, and with $M_u := \max_{j} \|\mathbf{U}^j\|$, $M_x := \max_{j} \|(\mathbf{X}_0)^j\|$ being the maximum norms of the control and imposed token embeddings, respectively.

\end{theorem}

\begin{remark}
\label{rem:convergence}
The upper bound $k\gamma_i(\mathbf X_0, \boldsymbol{\theta})$ scales linearly with $k$, implying that the set of unreachable $\mathbf Y^*$ becomes smaller as $k$ grows larger.
Moreover, $\gamma$ is solely a function of the imposed state $\mathbf X_0$.
\end{remark}

\paragraph{Proof Summary: } An important idea of the proof is the decomposition of the output representations $\mathbf{Y}$ into two components: $\mathbf{Y}_x$ and $\mathbf{Y}_u$. The $\mathbf{Y}_x$ component arises from the value projections of the imposed state $\mathbf{X}_0$, while $\mathbf{Y}_u$ arises from the value projections of the control input $\mathbf{U}$. 
Although the softmax operation in the self-attention mechanism introduces cross-terms between $\mathbf{X}$ and $\mathbf{U}$ in both $\mathbf{Y}_x$ and $\mathbf{Y}_u$, we can disentangle their influences by considering the auxiliary representations $\mathbf{Y}_x^{\operatorname{max}}$ and $\mathbf{Y}_u^{\operatorname{max}}$. 
Specifically, $\mathbf{Y}_x^{\operatorname{max}}=\Xi(\mathbf X_0; \boldsymbol{\theta})$ represents the output of the self-attention mechanism $\Xi$ when only the imposed state $\mathbf{X}_0$ is provided as input, without any control input $\mathbf{U}$.  
We derive the bound in Theorem~\ref{thm:attention-control} by first deriving the bound $\beta_i  \geq \|\mathbf Y_u^i\|$ on row $i$ of $\mathbf Y_u$. In Appendix~\ref{app:beta_to_gamma}, we observe that, if $\|\mathbf Y_{x,\perp}^i\| \geq \beta_i$, it is impossible for $\|\mathbf Y_{u,\perp}^i\|$ to nullify the orthogonal component of $\mathbf Y_x$, rendering $\mathbf Y^*$ unreachable. A simplification of this inequality yields our bound $\|\mathbf{Y}_{x,\perp}^{\operatorname{max},i}\| > k \gamma_i(\mathbf X_0, \boldsymbol{\theta})$.

\paragraph{Discussion of Theorem~\ref{thm:attention-control}: } The reachable set exclusion condition in Equation~\eqref{eqn:attention-reachability-condition} arises when the output representation $\mathbf{Y}_x^{\operatorname{max}}$, which depends only on the imposed state $\mathbf{X}_0$, is too far away from the desired output $\mathbf{Y}^*$ for the control input $\mathbf{U}$ to steer the output towards $\mathbf{Y}^*$. The ability of the control input $\mathbf{U}$ to nullify the impact of $\mathbf{Y}_x^{\operatorname{max}} = \Xi(\mathbf{X}_0; \boldsymbol{\theta})$ scales with the number of control tokens $k$ (see hyperbolic relationship in Equation \ref{eq:beta_defn}). A longer control input can "dominate" the influence of $\mathbf{X}_0$ by increasing the relative contribution of $\mathbf{Y}_u$ to the overall output $\mathbf{Y}$. 

Furthermore, the proof reveals that the output of self-attention can be decomposed into components that depend primarily on different parts of the input (i.e., $\mathbf{X}_0$ and $\mathbf{U}$). 
While there are cross-terms in the attention matrix $(\mathbf X_0)^{i}\mathbf W_q\mathbf W_\mathrm{key}^\top (\mathbf X_0)^{j\top}$, these only introduce positive scaling factors (e.g., functions of $g_i$) to components (e.g. $\mathbf Y_x$, $\mathbf Y_u$) that are not dependent on the control input, allowing us to derive an analytic bound on the reachable output set for self-attention via $\mathbf Y_x^{\operatorname{max}}$ (see Equations~\ref{eq:bound}-\ref{eq:simplified_Ymin_def},\ref{eq:equivalent}). 

The implications of Theorem~\ref{thm:attention-control} are further discussed in Section~\ref{sec:discussion}. See Appendix~\ref{app:proof-control-llms} for proofs, including Section~\ref{app:more_general_theorem} for a more general statement of reachability conditions in terms of the perpendicular and orthogonal components of $\mathbf Y^*$ and $\mathbf Y$.

\section{Experiments}

To gain a practical, empirical understanding of the reachable set $\mathcal R_y^k(\mathbf x_0)$, we probe the existence of optimal prompts $\mathbf u^*$ across datasets $\mathcal D$ of initial state--desired output pairs $(\mathbf x_0, y^*)$. 
We scope our experiments to study immediate control (i.e., we check the LLM output after $|y^*|$ tokens are generated) where the control input $\mathbf u$ is prepended to the imposed state $\mathbf x_0$. 
Moreover, we focus on the case of controlling the LLM system to produce a single output token $y^* \in \mathcal V$ under some constraint $|u| \leq k$.
This ``single-step'' control renders the problem of gauging reachability computationally tractable and is a fundamental step toward understanding the iterated dynamics of LLM systems in terms of reachability and controllability. 
We leave the exploration of reachability and controllability under an extended time horizon (e.g., chain-of-thought, chatbot dynamics, tool-wielding LLMs) and under the requirement of multi-token outputs $\mathbf y$ to future work.

\label{sec:experiments}
\subsection{Methods}
\label{sec:methods}

We apply prompt optimization algorithms to establish the existence of optimal prompts $\mathbf u^*$ of length $k$ that steer the LLM system from initial state $\mathbf x_0$ to output $y$ for some dataset $\mathcal D$ of initial state-output pairs. 
In general, prompt optimization algorithms accept a token sequence and a loss function on said token sequence, along with a specification of which tokens are manipulable. 
The output of a prompt optimizer is a manipulated token sequence (i.e., optimized prompt) designed to minimize the loss. 
We apply two computational methods to generating optimal prompts: greedy back-generation (algorithm \ref{alg:greedy}) and greedy coordinate gradient (GCG, invented in \cite{zou2023universal}, algorithm \ref{alg:gcg}). 
We found that greedy back-generation performed best for short prompts $k\leq 3$ tokens, while GCG was the best-performing algorithm for prompts of 4 or more tokens. 
To our knowledge, our greedy back-generation algorithm is novel. 
For brevity, we place the full description of the algorithms and our parameter values for the two algorithms in Appendix~\ref{sec:algs}, as the specifics of the algorithms are not the main contribution of this work.

We focus on understanding the content and structure of the reachable set of LLM system outputs $\mathcal R_y^k(\mathbf x_0)$, particularly under a constraint on the number of input tokens $k$. 
To determine which output tokens are reachable under varying input sequence lengths, we apply an incremental prompt lengthening procedure when searching for optimal prompts on some dataset $\mathcal D$. 

\begin{algorithm}[tb]
\caption{Back-off Prompt}
\label{alg:back-off-prompt}
\begin{algorithmic}[1]
\REQUIRE State-output token sequence $(\mathbf{x}0, y)$; LLM system $\Sigma = (P{LM}, \mathcal{V})$.
\FOR{$k=1$ \TO $3$}
\STATE $\mathbf{u}_k = \text{Greedy Back Generate}(\mathbf{x}_0, y; \Sigma)$
\IF{$\mathbf{u}_k$ steers $\Sigma$ from $\mathbf{x}_0 \to y$}
\RETURN $\mathbf{u}_k$
\ENDIF
\ENDFOR
\FOR{$k \in {4, 6, 8, 10}$}
\STATE $\mathbf{u}_k = \text{Greedy Coordinate Gradient}(\mathbf{x}_0, y; \Sigma)$
\IF{$\mathbf{u}_k$ steers $\Sigma$ from $\mathbf{x}_0 \to y$}
\RETURN $\mathbf{u}_k$
\ENDIF
\ENDFOR
\RETURN Failed to establish reachability.
\end{algorithmic}
\end{algorithm}

\subsection{Results}
\label{sec:results}

Our results revolve around the reachable set $\mathcal R_y^k(\mathbf x_0)$ for state sequences sampled from the Wikitext dataset. 
Results were computed for a panel of models, including Falcon-7b, Falcon-40b, and Llama-7b. 
Falcon-7b results are showcased in this section while additional plots and results for Falcon-40b and Llama-7b can be found in Section~\ref{sec:sup_figs}.
We applied the same Back-off Prompt strategy (Algorithm~\ref{alg:back-off-prompt}) to determine $k$-$\epsilon$ controllability for all experiments, varying the specifics of the dataset $\mathcal D$ for each experiment.

\paragraph{``Ground truth'' reachability: } We established the reachability of the ``ground truth'' next token $y$ proceeding state token sequence $\mathbf x_0$ in Wikitext. 
In our tests on a dataset of 5000 state-output sequences with states of length $8-32$ tokens, we found that the true next token $y$ is reachable over 97\% of the time across all models with a prompt of length $k\leq 10$ (Figure~\ref{fig:falcon_7b_k-e}). Plots and supplementary figures for Falcon-40b and Llama-7b controllability w.r.t. ground truth Wikitext outputs can be found in Section~\ref{sec:sup_main}.

\begin{figure}[ht]
    \centering
    \begin{minipage}{0.45\columnwidth}
        \includegraphics[width=\linewidth]{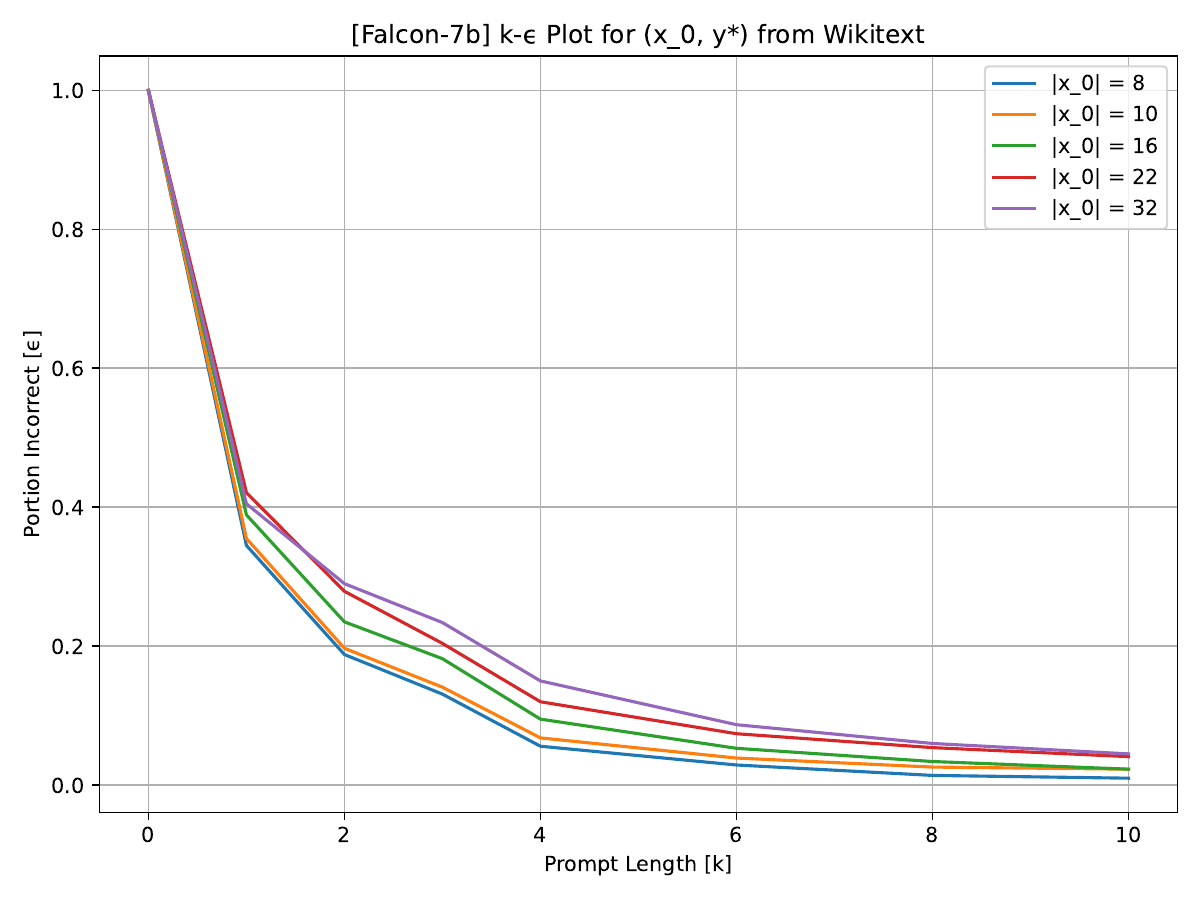}
    \end{minipage}%
    \hfill
    \begin{minipage}{0.45\columnwidth}
        \includegraphics[width=\linewidth]{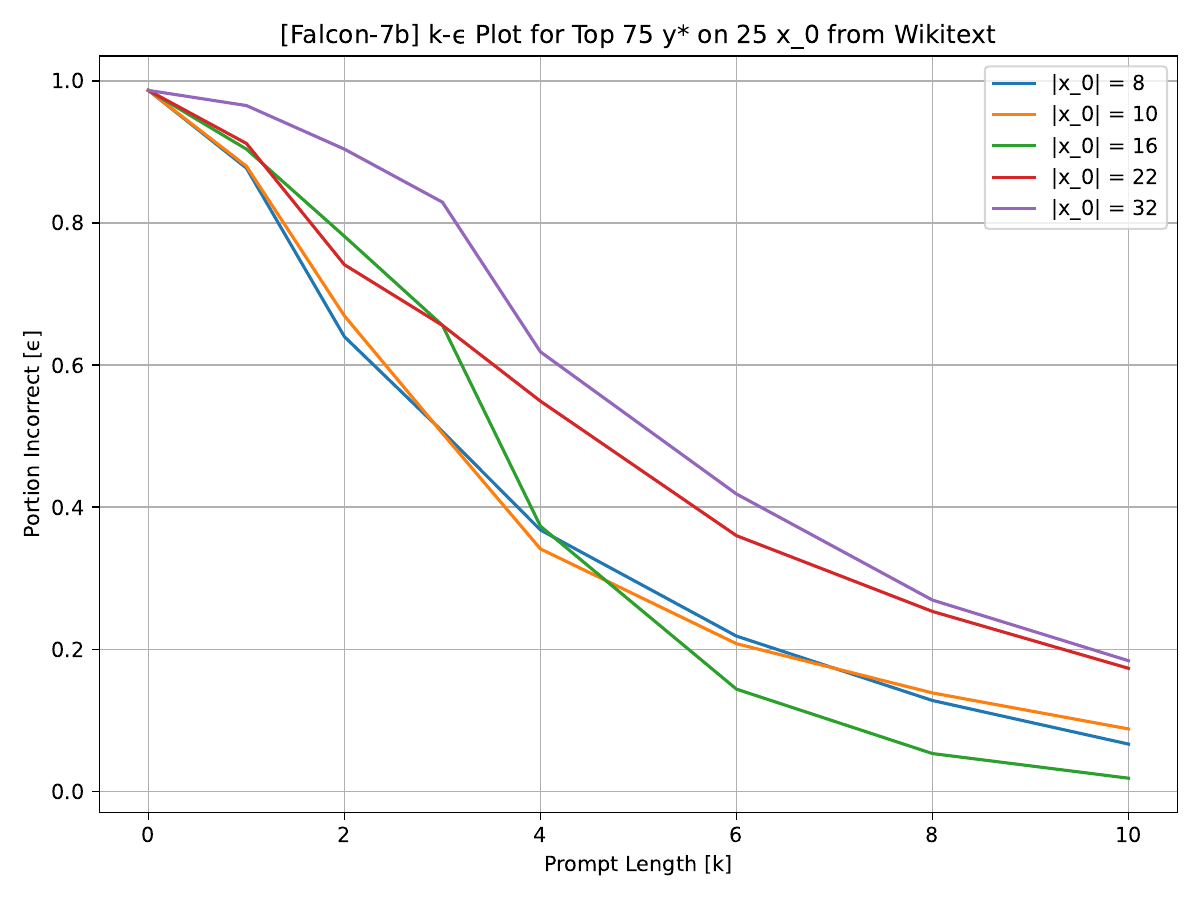}
    \end{minipage}
    
    \vspace{0.5cm} %
    
    \begin{minipage}{0.45\columnwidth}
        \includegraphics[width=\linewidth]{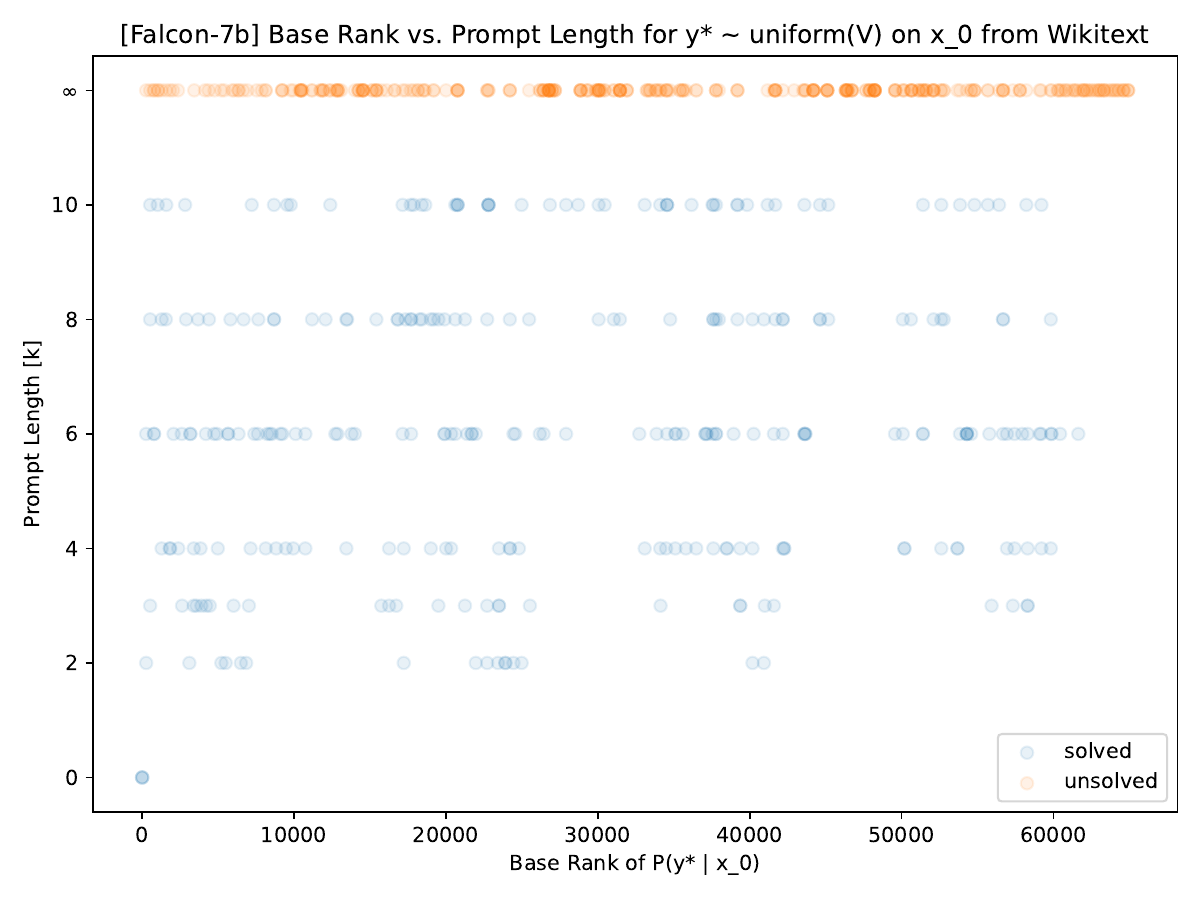}
    \end{minipage}%
    \hfill
    \begin{minipage}{0.45\columnwidth}
        \vspace{0pt} %
        \caption{
            \label{fig:falcon_7b_k-e}\textbf{Top Left}: $k$-$\epsilon$ values on initial state $\mathbf{x}_0$ and target output token $y^*$ from Wikitext. 97.16\% of the instances were solved with a prompt of length $k\leq 10$.\\
            \label{fig:falcon_7b_top75}\textbf{Top Right}: $k$-$\epsilon$ values reaching the top 75 most likely outputs $y^*$ for each $\mathbf{x}_0$ from Wikitext. The top 75 targets were reachable at least 89.39\% of the time with a prompt of length $k\leq 10$.\\
            \label{fig:deep_falcon_rank_k}\textbf{Bottom Left}: Prior likelihood rank of target token $y^*$ versus required prompt length to elicit $y^*$. Target tokens were sampled uniformly from the least to most likely token given $\mathbf{x}_0$ sampled from Wikitext.
        }
    \end{minipage}
\end{figure}

\paragraph{Top-75 reachability: } To explore the reachable set $\mathcal R_y^k(\mathbf x_0)$ beyond the ground truth of Wikitext outputs, we generated a synthetic dataset of outputs by sampling 25 Wikitext sequences $\mathbf x_0$ and selecting the top 75 most likely next-tokens according to the model itself $P_{LM}(y | \mathbf x_0)$ as the target tokens (Figure~\ref{fig:falcon_7b_top75}). We found that the top 75 output tokens were reachable over 85\% of the time for all models with control sequence length $k=10$. Supplementary figures including results for Llama-7b and Falcon-40b on $k$-$\epsilon$ controllability with respect to the top 75 most likely output tokens can be found in Section~\ref{sec:sup_75}.

\paragraph{Uniformly sampled target outputs: } To maximally push the bounds of the reachable set within our single output token scope, we created another synthetic dataset where the target output token $y^*$ was sampled uniformly from the highest likelihood next token to the lowest likelihood token. 
Although the overall $k$-$\epsilon$ score was relatively poor (only 46.43\% reachable with $k=10$ for Falcon-7b), we were intrigued by the near-uniform relationship between prior token rank (based on $P_{LM}(y | \mathbf x_0)$) versus the required number of prompt tokens. 
Figure~\ref{fig:deep_falcon_rank_k} plots the relationship between prior target token rank based on $P(y^* | \mathbf x_0)$ and the required prompt length $k$ to elicit the prompt. While over half were unreachable, the remaining reachable tokens appear uniformly distributed in terms of required prompt length, regardless of rank. Supplementary figures analyzing the $k$-$\epsilon$ controllability of Falcon-7b with respect to uniformly sampled target outputs $y$ can be found in Section~\ref{sec:sup_deep}.

\section{Discussion}
\label{sec:discussion}

We proposed a control theoretic framework for understanding language model prompting, orienting our investigation around the reachable set of outputs $\mathcal R_y^k(\mathbf x_0)$.
We proved a bound on the reachable set of outputs for self-attention in terms of the singular values of its weight matrices, and we established fundamental results on the reachability of ``correct'' next tokens (according to Wikitext). We expanded the scope of this investigation by probing the reachability of tokens assigned high likelihood by the LLM itself (top 75 most likely next tokens), and tokens assigned minimal likelihood by the LLM itself (randomly sampled target tokens).

The Self-Attention Control Theorem (Theorem~\ref{thm:attention-control}) provides a sufficient condition for the unreachability of a desired output $\mathbf{Y}^*$ in terms of the projection of a single row of $\mathbf Y^{max}_{x} = \Xi(\mathbf X_0; \boldsymbol{\theta})$ onto the orthogonal complement of $\mathbf{Y}^*$. 
If the orthogonal component of $\mathbf Y^{\rm max}_x$ exceeds $k\gamma$, then no prompt of length $\leq k$ can steer the self attention head to output $\mathbf Y^*$ under the input constraints. 
The threshold $k \gamma_i(\mathbf{X}_0, \boldsymbol{\theta})$ depends on the imposed input $\mathbf{X}$, the number of control tokens $k$, and the maximum singular values of the query, key, and value weight matrices, $\boldsymbol{\theta}=(\mathbf W_k, \mathbf W_q, \mathbf W_v)$. 
Intuitively, this result suggests that if the output $\mathbf{Y} = \mathbf{Y}_x+\mathbf{Y}_u$ has component $\mathbf Y_x$ too large and misaligned with $\mathbf Y^*$, then no control input with $k$ or fewer tokens can yield a component $\mathbf Y_u$ that corrects the misalignment -- even if control inputs $\mathbf U$ yield maximal influence on the output under the $k$-token limit (Figure~\ref{fig:attention_figure}`).

Bounding the reachable set for self-attention is deeply related to the mechanism by which consistent representations are formed for multi-token generation.
Steering a language model to generate a desired token sequence requires that the control input induce a token representation in the right-most token such that the next token prediction logits $P(\mathbf y | \mathbf u + \mathbf x_0)$ achieves a desired value. 
Moreover, generated tokens are fed back into the model, and their representations must be steered as well to control iterated generation.
Self-attention is the primary mechanism by which the token representations exchange information, making the reachable set of output representations across multiple tokens in $\mathbf X_0$ for self-attention a fundamental part of LLM control theory.  The Self-Attention Control Theorem provides a step towards understanding the limitations and possibilities of controlling the self-attention layer, and by extension, the language model as a whole.

Our empirical results suggest that there is far more to the reachability of a given output than just prior likelihood or the prior rank the LLM assigns to a given token. 
Although prompt optimization-based $k$-$\epsilon$ controllability experiments are only able to provide a lower bound on the content of the reachable set, the ability to frequently control even the \textit{least likely} token to being the \textit{most likely} token with just a few input tokens is intriguing (Figure~\ref{fig:deep_falcon_rank_k}, bottom right). 
This result indicates the importance of further investigating the reachability and controllability of LLMs, particularly for developing capable and reliable LLM systems.

Our investigations provide an entry into the understanding of LLM controllability via prompts. 
However, a comprehensive understanding necessitates extending our exploration into diverse regimes. 
Exploring the controllability with longer prompts and longer questions (base token sequences) will be pivotal. 
Equally important is the study of diverse models to verify the generality of our findings. 
The direct comparison of controllability scores of different model families is challenging since each model family uses a different tokenizer. The Llama family tokenizer, for instance, has a vocabulary of 30,000 tokens whereas the Falcon family has a vocabulary of 65,536 tokens. Further work is required to robustly compare controllability across models.

An intriguing observation from our study is the log-linear relationship between prompt length $k$ and controllability fraction $\epsilon$ (see Figure~\ref{fig:log_main} in Appendix~\ref{sec:sup_figs}). 
While this is compelling within our studied domain, it raises the essential question: is this relationship robust outside our current explorative scope? 
Unearthing universal scaling laws in LLM controllability would not only inform practical control applications but also open the door for theoretical insight into the nature of LLM behavior.

The progress we have made, both in understanding the bounds on self-attention controllability and the empirical measures of $k$-$\epsilon$ LLM controllability, underscores the potential of this control theoretic framing for studying LLMs. 
Below is a non-exhaustive list of open problems in LLM control, all stemming from the framing in section~\ref{sec:definitions_ctrl}: 
\begin{itemize}
    \item \textbf{Control Properties of Chain-of-Thought:} Chain-of-Thought is a powerful technique where LLMs are allowed to generate intermediate tokens (i.e., ``thoughts'') between a question and an answer \citep{wei2023chainofthought}. 
    The control properties (e.g., stability, reachability) of systems leveraging these techniques are of great interest for understanding and composing systems of LLMs in the real world.  
    \item \textbf{Distributional Control:} To what extent can we control the output distribution of a language model $P_{LM}(\mathbf y | \mathbf x_0 + \mathbf u)$ to a desired distribution $P^*(\mathbf y)$?
    \item \textbf{Computational Cost of Control:} What are the performance characteristics of LLM control regularized by computational cost? 
    \item \textbf{Learnability of Control:} To what extent can LLMs learn to control each other? Work such as \cite{zhou2023large} showed that LLMs are capable of human-level prompt engineering, but it is unclear how well an LLM can learn to control another when explicitly optimized on the objective of LLM control. 
    \item \textbf{Controllable Subspaces:} In the control of linear dynamical systems, it is known that uncontrollable systems are often coordinate transformable into a representation where a subset of the coordinates are controllable and a subset are uncontrollable \cite{control_bible}. We have shown that controllable and uncontrollable components naturally emerge for self-attention heads in section~\ref{sec:theorem} -- can this be generalized to transformer blocks with nonlinearities and residual streams? 
    \item \textbf{Composable LLM Systems:} One of the greatest boons of control theory is the ability to compose control modules and subsystems into an interpretable, predictable, and effective whole \citep{lian2002network}. 
    The composition of LLM systems (potentially with non-LLM control modules) is an exciting avenue for scaling super intelligent systems.
\end{itemize}

Practically, our findings lay the groundwork for more effective and efficient prompt engineering. The ability to control even the least likely tokens illuminates untapped capabilities within LLMs, hinting at a potentially broader spectrum of application than previously recognized. Such insights could lead to the development of more nuanced and sophisticated LLM systems, capable of handling tasks with greater precision and adaptability.

\newpage
\section*{Impact statement}

This paper introduces foundational work aimed at enhancing our understanding and control of generative language models (LLMs) as they become integral to crucial societal functions. 
The increasing integration of generative AI into critical infrastructures — such as healthcare data analysis, insurance and financial data processing, and emergency response systems — underscores the urgency for a sophisticated control theory. 
Drawing on the principles of control theory, which have historically ensured the dependability of machines in life-or-death scenarios (e.g., in cruise control and aircraft navigation systems), our goal is to extend these guarantees to LLM-based applications. 
By doing so, we aim to make these advanced AI systems as trustworthy and robust as their electro-mechanical counterparts, thereby securing their role in supporting and safeguarding society.

\section*{Code availability}

All code used to produce the experimental results is provided with the submission.

\bibliography{bhargava}
\bibliographystyle{ieeetr}

\newpage
\appendix
\onecolumn
\section{Abstract Systems and Control Theory Background}
\label{sec:definitions_ctrl}

This section aims to provide an overview of fundamental control-theoretic concepts from an abstract, set-theoretic perspective. We primarily draw from canonical textbooks \cite{control_bible, kalman1969topics}, and \cite{ogata2010modern}.

Diverse definitions of ``system'' or ``machine'' exist in the literature, all representing the same core concept but varying in mathematical details. We offer the following high-level definition based on \cite{control_bible} Chapter 2: 

\begin{definition}[System] \label{def:system}
A ``system'' or ``machine'' $\Sigma = (\mathcal{T, X, U}, \phi)$ consists of: 
\begin{itemize}
    \item $\mathcal T:$ The \textbf{time set} along which system state evolves. 
    \item $\mathcal X: $ The \textbf{state space}.
    \item $\mathcal U:$ The \textbf{input space}.
    \item $\phi: \mathcal{X \times U \times T}^2 \to \mathcal X: $ The \textbf{transition map}. 
\end{itemize}
A system may also be equipped with an output space and readout map $(\mathcal Y, h)$: 
\begin{itemize}
    \item $\mathcal Y:$ The \textbf{output space}. 
    \item $h: \mathcal{X \times U \times T}\to \mathcal Y: $ The \textbf{readout map}.
\end{itemize}
\end{definition}
In other words, at time $t\in \mathcal T$, the system's state takes on values $x \in \mathcal X$, and the control input takes values $u \in \mathcal U$. The system evolves over time with the transition map $\phi(x, u, t, t')$ that returns the new state value $x'\in \mathcal X$ at time $t'>t$. A system can also have a readout map $h(x, u, t)$ that produces the output value $y\in \mathcal Y$ given the current time, state, and input value. An input $u\in \mathcal U$ defined over interval $[t, t']$ may be said to \textit{steer the system} $\Sigma = (\mathcal{T, X, U}, \phi)$ from state $x_0$ to state $x'$ if $x' = \phi(x_0, u, t, t')$. A wide variety of systems are expressible within this framework. E.g., we obtain discrete-time dynamical systems for $\mathcal T = \mathbb Z^+$. Continuous-time dynamical systems emerge for $\mathcal T = \mathbb R^+$. 
We apply Definition~\ref{def:system} to formulate LLM systems in Definition~\ref{def:llm-system}. 

Note that we assume that the system $\Sigma$ is time-invariant; its dynamics $\phi$ do not change as a function of time. This assumption is widely applicable and is often made in the literature \citep{kalman1969topics, ogata2010modern, control_bible} to simplify definitions and discussions of systems.

Reachability is a core control theory concept and is central to defining controllability. At their core, definitions of reachability revolve around the existence of control inputs $u\in \mathcal U$ that steer the system from a starting state $x_0 \in \mathcal X$ to some desired state(s). Following from Chapters 1-2 of \cite{kalman1969topics} and Chapter 2 of \cite{control_bible}, we define state reachability as: 

\begin{definition}[State Reachability]
\label{def:state-reachability}
State $x \in \mathcal X$ is reachable from initial state $x_0\in \mathcal X$ for system $\Sigma=(\mathcal{T, X, U}, \phi)$ iff there exists some time $T$ and control input $u^* \in \mathcal U$ such that $u^*$ steers the system from state $x_0$ to state $x$ at time $T$.
\end{definition}

We may use this definition of state reachability to define the reachable state set for some initial state $x_0 \in \mathcal X$:

\begin{definition}[Reachable State Set]
\label{def:reachable-state-set} 
The reachable state set from initial state $x_0 \in \mathcal X$ for system $\Sigma = (\mathcal {T, X, U}, \phi)$ is denoted $\mathcal R(x_0) \subseteq \mathcal X$ and consists of all reachable states $x\in \mathcal X$ from initial state $x_0$ (cf. Definition~\ref{def:state-reachability}). 
\end{definition} 

For systems with readout maps $h$, notions of \textit{output reachability} arise naturally. Note that state reachability is neither necessary nor sufficient to guarantee output reachability. 

\begin{definition}[Output Reachability]
\label{def:output-reachability}
Output $y \in \mathcal Y$ is reachable from initial state $x_0 \in \mathcal X$ for system $\Sigma=(\mathcal{T, X, U}, \phi, \mathcal Y, h)$ iff there exists some time $T$ and control input $u^* \in \mathcal U$ such that $u^*$ steers the system from state $x_0$ to output $y$ in time $T$. 
\end{definition}

\begin{definition}[Reachable Output Set]
\label{def:reachable-output-set}
The reachable output set from initial state $x_0 \in \mathcal X$ for system $\Sigma = (\mathcal{T, X, U}, \phi, \mathcal Y, h)$ is denoted $\mathcal R_y(x_0)$ and consists of all reachable outputs $y\in \mathcal Y$ from initial state $x_0$ (cf. Definition~\ref{def:output-reachability}). 
\end{definition}

A system is controllable when the reachable set extends to the entire state space. Practically speaking, this implies that one can steer the system from any initial state to any desired state. 
We develop the reachable set for LLM systems in Definition~\ref{def:llm-reachable-output-set} and LLM reachability in Definition~\ref{def:llm-output-reachability}.

\begin{definition}[State Controllability]
\label{def:state-controllability} 
System $\Sigma = (\mathcal{T, X, U}, \phi)$ is state controllable iff, for every initial state $x_0\in \mathcal X$, the reachable set $\mathcal R(x_0) = \mathcal X$.
\end{definition}

\begin{definition}[Output Controllability] 
\label{def:output-controllability}
System $\Sigma = (\mathcal{T, X, U}, \phi, \mathcal Y, h)$ is output controllable iff, for every initial state $x_0\in \mathcal X$, the reachable output set $\mathcal R_y(x_0) = \mathcal Y$.
\end{definition}

A range of fruitful questions stem from these definitions: if there is a cost associated with control inputs $u\in \mathcal U$ (e.g., power constraints, length constraints), what is the minimum cost of control? What is the minimum time required to get from the initial state to the desired final state or output? If the system is not completely controllable, under what conditions is it controllable? Under which readout maps is a system output controllable? 
We develop controllability for LLMs abstractly in Definition~\ref{def:llm-output-controllability} and in an empirically/statistically testable fashion in Definition~\ref{def:k-eps}.

\section{Theory on Self-Attention Controllability}
\label{app:proof-control-llms}
\textit{Note: Key terms for the proof are introduced in Section~\ref{sec:theorem} surrounding Theorem~\ref{thm:attention-control}. Specifically, the definition of self-attention mechanism $\Xi$, the control problem setup, and the reachable set $\mathcal R_y^k(\mathbf X_0)$ are required background for this proof.}

\textbf{Notation:} For each token representation matrix $\mathbf{Q, K, V} \in \mathbb R^{(k+m) \times \cdot}$, we denote the first $k$ rows corresponding to $\mathbf U$ using $u$ as a subscript, like $\mathbf Q_u$. The remaining $m$ rows corresponding to $\mathbf X_0$ are denoted with subscript $x$ like $\mathbf Q_x$. 

\subsection{Proof of Theorem~\ref{thm:attention-control}}
\label{app:main_proof}
    Let $\mathbf A$ be the exponentiated query-key outer product matrix with the following block structure: 
    \begin{equation}
        \label{eqn:def-A}
        \mathbf A = \exp \begin{pmatrix}
            \frac{
                \textbf{Q K}^\top
            }{
                \sqrt{d_{\rm key}}
            }
        \end{pmatrix}
        = \exp \begin{pmatrix} 
            \begin{bmatrix}
                \mathbf{Q}_u \mathbf{K}_u^\top & \mathbf{Q}_u \mathbf{K}_x^\top \\
                \mathbf{Q}_x \mathbf{K}_u^\top & \mathbf{Q}_x \mathbf{K}_x^\top
            \end{bmatrix}
            \frac{1}{\sqrt{d_{\rm key}}}
        \end{pmatrix}
         = \begin{bmatrix}
             \mathbf A_{uu} & \mathbf A_{ux} \\ 
             \mathbf A_{xu} & \mathbf A_{xx} 
         \end{bmatrix}
    \end{equation}
    where $\mathbf Q_u = \mathbf U \mathbf W_q$, $\mathbf K_x = \mathbf X_0 \mathbf W_{\rm key}$, and similarly for $\mathbf K_u, \mathbf Q_x$. 
    We apply a similar quadrant decomposition to $\mathbf{D}$, defined initially in Equation~\ref{eqn:denominator-definition}.
    \begin{equation}
        \label{eqn:denom-decomp}
        \mathbf{D} = \text{diag}\begin{pmatrix} \exp\begin{pmatrix} \frac{\mathbf{QK}^\top}{\sqrt{d_{\rm key}}} \end{pmatrix} \mathbf 1_{N\times 1} \end{pmatrix}
        = \begin{bmatrix}
            \mathbf{D}_u & \mathbf 0 \\ 
            \mathbf 0 & \mathbf{D}_x \\ 
        \end{bmatrix}
    \end{equation}

    where the quadrant demarcations in $\mathbf{D}$ follow from Equation~\ref{eqn:def-A}. 

    We may now express the self-attention mechanism output representations $\mathbf Y$ as 
    \begin{equation}
        \label{eqn:y-decomp}
        \mathbf Y = \underbrace{\mathbf{D}_x^{-1} \mathbf A_{xu} \mathbf V_u}_{\mathbf Y_u} + \underbrace{\mathbf{D}_x^{-1} \mathbf A_{xx} \mathbf V_x}_{\mathbf Y_x}
    \end{equation}

\begin{lemma}
For any control input $\mathbf{U}$ whose rows satisfy $\|\mathbf{U}^j\| \leq M_u$ for all $j \in \{1, \dots, k\}$, the norm of the $i$-th row of $\mathbf Y_u$ is bounded as follows
    $$\|\mathbf{Y}_u^i\| \leq \beta_i(\mathbf{X}_0, k)$$ where
    \begin{equation}
    \beta_i(\mathbf{X}_0, k):=\frac{ ke^{\alpha}}{g_i(\mathbf X_0,\boldsymbol{\theta}) + ke^{\alpha} } \sigma_v M_u, 
        \label{eq:beta_defn}
    \end{equation} and $$\alpha = \sigma_q \sigma_{\rm key} M_u M_x / \sqrt{d_{\rm key}}, \qquad  g_i(\mathbf{X}_0,\boldsymbol{\theta}) := \mathbf{D}_{xx}^i=\sum_{j=1}^m \exp\left( (\mathbf X_0)^i \mathbf{W}_q \mathbf{W}_{\rm key}^\top (\mathbf X_0)^{j\top}/\sqrt{d_{\rm key}}\right).$$ 
\end{lemma}
\begin{proof}
Our objective is to establish an upper bound on $\|\mathbf{Y}_u^i\|$, the Euclidean norm of the $i$-th row of the matrix $\mathbf{Y}_u$, which represents the contribution of the control input to the output of the self-attention layer. 
$g_i = \mathbf D_{xx}^i$ represents the component of the row-wise softmax denominator $\mathbf D_x$ from $\mathbf A_{xx}$ (solely a function of $\mathbf X_0$). 
Similarly, $\mathbf D_{xu}$ represents the component of $\mathbf D_x$ from $\mathbf A_{xu}$, and $\mathbf D_x = \mathbf D_{xx} + \mathbf D_{xu}$. 
We observe that $\mathbf{D}_{xu}^i$ is the sum of the entries in the $i$-th row of $\mathbf{A}_{xu}$:
\begin{equation}
\mathbf{D}_{xu}^i = \sum_{j=1}^k (\mathbf{A}_{xu})_{ij} = \sum_{j=1}^k \exp(\langle \mathbf{Q}_x^i, \mathbf{K}_u^j \rangle / \sqrt{d_{\rm key}}),
\end{equation}
where $\mathbf{Q}_x^i$ and $\mathbf{K}_u^j$ denote the $i$-th row of $\mathbf{Q}_x$ and the $j$-th row of $\mathbf{K}_u$, respectively. Every entry of the diagonal matrix $\mathbf{D}_{xu}$ is strictly positive.

Recall that $\mathbf D_x=\mathbf D_{xx}+\mathbf D_{xu}$ and $\mathbf V_u=\mathbf U \mathbf W_v$. We begin by expressing the $i$-th row of $\mathbf{Y}_u$ as:
\begin{equation}
\mathbf{Y}_u^i = (\mathbf{D}_{xx}^i + \mathbf{D}_{xu}^i)^{-1} (\mathbf{A}_{xu}^i \mathbf{V}_u^i),
\end{equation}
where $\mathbf{D}_{xx}^i$ and $\mathbf{D}_{xu}^i$ denote the $i$-th diagonal entries of the matrices $\mathbf{D}_{xx}$ and $\mathbf{D}_{xu}$, respectively, $\mathbf{A}_{xu}^i$ represents the $i$-th row of the matrix $\mathbf{A}_{xu}$, and $\mathbf{V}_u^i$ corresponds to the $i$-th row of the matrix $\mathbf{V}_u$.

Let $\alpha_{ij} := (\mathbf A_{xu})_{ij} = \langle \mathbf{Q}_x^i, \mathbf{K}_u^j \rangle / \sqrt{d_{\rm key}} \leq \alpha$ for all $i,j$ where $\alpha$ is defined to be an upper bound on the scaled key-query dot products between vectors in $\mathbf U$ and $\mathbf X$ given by 
\begin{equation}
    \label{eqn:alpha_defn}
    \alpha = \sigma_q \sigma_{\rm key} M_u M_x / \sqrt{d_{\rm key}}.
\end{equation}
Recall that $\sigma_q, \sigma_{\rm key}$ are the maximal singular values of $\mathbf W_q, \mathbf W_{\rm key}$ respectively. 
 
By applying the Cauchy-Schwarz inequality and using the definitions of $\alpha$, $M_u$, and $M_x$, we can perform the bound $(\mathbf{A}_{xu})_{ij} \leq e^\alpha$ for all $i,j$, and thus:
\begin{equation}
\mathbf{D}_{xu}^i = \mathbf{A}^i_{xu}\mathbf 1 \leq \sum_{j=1}^k e^\alpha = ke^\alpha.
\label{eq:ke_alpha}
\end{equation}
where $\mathbf{1}$ is a constant vector consisting of all entries equal to 1.

Next, we note that $\|\mathbf{V}_u^i\| \leq C$, where $C := \sigma_v M_u$, and $\sigma_v$ denotes the maximum singular value of the value projection matrix $\mathbf{W}_v$. That fact follows directly the definition of $\mathbf V_u$. This allows us, while we are bounding $\|\mathbf{Y}_u^i\|$, to replace $\mathbf{V}_u^i$ with a constant vector $\mathbf{C}$ whose entries are all equal to $C$, yielding an upper bound on $\|\mathbf{A}_{xu}^i \mathbf{V}_u^i\|$:
\begin{equation}
\|\mathbf{A}_{xu}^i \mathbf{V}_u^i\| \leq \|\mathbf{A}_{xu}^i \mathbf{C}\| = C \sum_{j=1}^k (\mathbf{A}_{xu})_{ij} = C \mathbf{D}_{xu}^i.
\end{equation}

We now rewrite the norm $\|\mathbf{Y}_u^i\|$. Toward that end, let $g_i(\cdot, \boldsymbol{\theta}) : \mathbb R^{m\times d}\to [0,\infty)$ denote the function of $\mathbf X_0$ defined by $g_i(\mathbf{X}_0, \boldsymbol{\theta}) := \mathbf{D}_{xx}^i$.
\begin{align}
\|\mathbf{Y}_u^i\| 
&= \frac{\|\mathbf{A}_{xu}^i \mathbf{V}_u^i\|}{\mathbf{D}_{xx}^i + \mathbf{D}_{xu}^i} \\
&\leq \frac{\|\mathbf{A}_{xu}^i\mathbf{C}\|}{\mathbf{D}_{xx}^i + \mathbf{D}_{xu}^i} = \frac{\langle \mathbf{A}_{xu}^i,\mathbf{C}\rangle}{\mathbf{D}_{xx}^i + \mathbf{A}_{xu}^i \mathbf{1}} \\
&\leq \frac{C ke^{\alpha}}{g_i + ke^{\alpha} }\label{eq:bound}
\end{align}
The final line follows from \eqref{eq:ke_alpha} and the observation that the function $f(x)  := x/(x+g_i)$, where $g_i>0$ is monotone increasing. 

Let 
\begin{equation}
    \beta_i(\mathbf{X}_0, k):=\frac{ ke^{\alpha}}{g_i(\mathbf X_0,\boldsymbol{\theta}) + ke^{\alpha} } \sigma_v M_u.
\end{equation} We have established that $\|\mathbf{Y}_u^i\| \leq \beta_i(\mathbf{X}_0, k)$ for any control input $\mathbf{U}$ whose rows satisfy $\|\mathbf{U}^j\| \leq M_u$ for all $j \in \{1, \dots, k\}$. The same bound holds for $\|\mathbf{Y}_{u,\perp}^i\|$, the norm of the projection of $\mathbf{Y}_u^i$ onto the orthogonal complement of $\mathbf{Y}^*$.
\end{proof}

\subsection{Simplified reachability hypothesis}
\label{app:beta_to_gamma}
We can restate the hypothesis of our self-attention theorem, Theorem~\ref{thm:attention-control} \begin{equation}
    \|\mathbf{Y}^{\operatorname{max}}_{x,\perp}\| > \frac{ke^\alpha}{g_i}\sigma_v M_u
    \label{eq:equivalent}
\end{equation}
as equivalent to
\begin{equation}
\|\mathbf{Y}^{\operatorname{min},i}_{x,\perp}\| > \beta_i = \frac{ke^\alpha}{ke^\alpha + g_i} \sigma_v M_u.
\label{eq:simplified_Ymin_bd}
\end{equation}
Since 
\begin{equation}
\mathbf{Y}^{\operatorname{min},i}_x = \frac{g_i}{g_i + ke^\alpha}\mathbf{Y}^{\operatorname{max},i}_x,
\label{eq:simplified_Ymin_def}
\end{equation}
and
$g_i>0$ and $ke^\alpha$ are positive scalars, we can cancel the factor of $(ke^\alpha+g_i)^{-1}$ on both sides, and then divide both sides by $g_i$, to obtain the equivalent hypothesis \eqref{eq:equivalent} from hypothesis \eqref{eq:simplified_Ymin_bd}.

\subsection{More general theorem }

\label{app:more_general_theorem}
\begin{theorem}[Self-Attention Control Theorem 2]
\label{thm:attention-control-improved}
Consider a self-attention layer with input $\mathbf{X} \in \mathbb{R}^{m \times d}$ and control input $\mathbf{U} \in \mathbb{R}^{k \times d}$, where $m$ is the number of imposed tokens, $k$ is the number of control tokens, and $d$ is the token embedding dimension. Let $\mathbf{Y}^* \in \mathbb{R}^{m \times d}$ be the desired output, and let $\mathbf{Y} \in \mathbb{R}^{m \times d}$ be the actual output of the self-attention layer. As before, define $\mathbf{Y}_{\perp} = \mathbf{Y}_{x,\perp} + \mathbf{Y}_{u,\perp} \in \mathbb{R}^{m \times d}$ as the projection of the output onto the orthogonal complement of $\mathbf{Y}^*$.

If either:

(A) $\|\mathbf{Y}\| = \|\mathbf{Y}^*\|$ and there exists a component $\mathbf{Y}_{\perp}^{ij} \neq 0$ of the matrix $\mathbf{Y}_{\perp}$, %
or

(B) $\|\mathbf{Y}\| \neq \|\mathbf{Y}^*\|$,

then

\begin{equation}
\mathbf{Y} \neq \mathbf{Y}^*
\end{equation}

for any control input $\mathbf{U} \in \mathbb{R}^{k \times d}$ such that $\max_j \|\mathbf{U}^j\| \leq M_u$.

\end{theorem}

This theorem is also illustrated in Figure~\ref{fig:attention_figure} and is a more general theorem than Theorem~\ref{thm:attention-control}: the hypothesis of Theorem~\ref{thm:attention-control} implies that some row satisfies $\|\mathbf Y^{\operatorname{min},i}_{x,\perp}\| >  \|\mathbf Y^{\operatorname{max},i}_{u,\perp}\|$, so it must be the case that there exists some nonzero entry $\mathbf{Y}^{ij}_\perp$ of the matrix $\mathbf{Y}_\perp$ in the case that $\|\mathbf{Y}\| = \|\mathbf{Y}^*\|$. %

The Self-Attention Control Theorem (Theorem \ref{thm:attention-control}) provides valuable insights despite being less general than the more general version (Theorem \ref{thm:attention-control-improved}). An advantage of Theorem \ref{thm:attention-control} is its more specific hypothesis, Equation \eqref{eqn:attention-reachability-condition}, which provides a concrete criterion for determining whether the desired output can be achieved by the self-attention layer\footnote{and depends on the properties of the input tokens, the control tokens, and the learned parameters of the self-attention layer, such as the maximum singular values of the query and key projection matrices}. 

\begin{proof}[Proof of Theorem~\ref{thm:attention-control-improved}]
We will prove the theorem by contradiction. Assume that $\mathbf{Y} = \mathbf{Y}^*$ for some control input $\mathbf{U}$ satisfying $\max_j \|\mathbf{U}^j\| \leq M_u$. 

Case (A): If $\|\mathbf{Y}\| = \|\mathbf{Y}^*\|$ and there exists a component $\mathbf{Y}_{\perp}^{ij} \neq 0$ of the matrix $\mathbf{Y}_{\perp}$, then $\mathbf{Y}_\perp \neq \mathbf{0}$. This implies that $\mathbf{Y}$ and $\mathbf{Y}^*$ are not parallel, and therefore $\mathbf{Y} \neq \mathbf{Y}^*$, contradicting the assumption.

Case (B): If $\|\mathbf{Y}\| \neq \|\mathbf{Y}^*\|$, then $\mathbf{Y} \neq \mathbf{Y}^*$ directly, again contradicting the assumption.

In both cases, we have a contradiction, so the assumption that $\mathbf{Y} = \mathbf{Y}^*$ must be false, and we can conclude that $\mathbf{Y} \neq \mathbf{Y}^*$ for any control input $\mathbf{U}$ satisfying $\max_j \|\mathbf{U}^j\| \leq M_u$.

To show that Theorem~\ref{thm:attention-control} is a special case of the more general theorem, consider the hypothesis of Theorem~\ref{thm:attention-control}: from \eqref{eqn:attention-reachability-condition} we conclude that some row satisfies $\|\mathbf{Y}^{min,i}_{x,\perp}\| > \|\mathbf{Y}^{max,i}_{u,\perp}\|$. This implies that $\mathbf{Y}_{x,\perp}^{ij} \neq -\mathbf{Y}_{u,\perp}^{ij}$ for some entry $(i,j)$, and therefore $\mathbf{Y}_{\perp}^{ij} = \mathbf{Y}_{x,\perp}^{ij} + \mathbf{Y}_{u,\perp}^{ij} \neq 0$. This satisfies the condition of case (A) in the more general theorem, assuming $\|\mathbf{Y}\| = \|\mathbf{Y}^*\|$. Thus, the hypothesis of Theorem~\ref{thm:attention-control} is a special case of the hypothesis in the more general theorem.
\end{proof}

By incorporating this bound into the hypothesis, Theorem \ref{thm:attention-control} offers a more practical and actionable result, allowing researchers and practitioners to assess the controllability of a self-attention layer based on measurable quantities, without the need to exhaustively search the space of possible control inputs. Moreover, the presence of the bound opens up opportunities for further analysis and optimization, potentially guiding the design of control strategies that satisfy the bound and ensuring that the desired output can be reached. Additionally, the bound can be used to derive insights into the relationship between the properties of the input tokens, the control tokens, and the achievable control over the self-attention layer's output.
While Theorem \ref{thm:attention-control-improved} provides a more general result, Theorem \ref{thm:attention-control} complements it by incorporating a specific bound involving $\gamma_i$ into its hypothesis. This specific bound in Theorem \ref{thm:attention-control}'s makes it more practical for control of self-attention layers in applications.

\subsection{Discussion}

As seen in equation \eqref{eq:beta_defn}, $\beta_i(\mathbf{X}_0, k)$ exhibits a hyperbolic dependence on $ke^\alpha$. This suggests that increasing the number of control tokens can ``dominate'' the output of the self-attention, overwhelming the influence of the imposed state sequence $\mathbf X_0$. 
The theorem's reachability condition depends on the number of control tokens $k$ through the threshold $\gamma_i(\mathbf{X}_0, \boldsymbol{\theta})$. As discussed in Remark~\ref{rem:convergence}, the threshold scales linearly with $k$, suggesting that increasing the number of control tokens can potentially enhance controllability. However, this effect is modulated by the other terms in the threshold, such as $\alpha$ and $g_i(\mathbf{X}_0,\boldsymbol{\theta})$, which depend on the properties of the imposed tokens and the model parameters. 
Specifically, $\beta_i$ saturates to 1 as $k\to \infty$ or as $\alpha$ becomes very large.

The term $g_i(\mathbf X_0,\boldsymbol{\theta})$ captures the influence of the imposed tokens on the attention weights and appears in the denominator of the threshold $\gamma_i(\mathbf{X}_0, \boldsymbol{\theta})$. Larger values of $g_i(\mathbf{X}_0,\boldsymbol{\theta})$ lead to a lower threshold, which may make reachability easier, thus increasing the potentially reachable set size. \\

The hypothesis of Theorem~\ref{thm:attention-control} implies that some row of the projection of the minimum possible output $\mathbf{Y}_x^{{\rm min}}$ onto the orthogonal complement of $\mathbf{Y}^*$ exceeds the corresponding row of the maximum possible projection of $\mathbf{Y}_u$. This ensures the existence of a non-zero component in $\mathbf{Y}_\perp$ and precludes reachability. Thus, Theorem~\ref{thm:attention-control} provides a more specific, practically applicable criterion for assessing controllability than Theorem~\ref{thm:attention-control-improved}.

Theorem~\ref{thm:attention-control}'s reachability condition depends on the maximum singular values of the query, key, and value projection matrices ($\mathbf{W}_q, \mathbf{W}_{\rm key}, \mathbf{W}_v$). %
The $\alpha$ term in Theorem~\ref{thm:attention-control}, which involves the maximum singular values of the query and key projection matrices, provides an upper bound on the scaled dot products that is only tight in the special case of maximal alignment between the query and key matrices. In general, the actual size of the threshold $\gamma_i$ will be smaller depending on $g_i$ and the alignment of queries from $\mathbf X_0$ with the keys from $\mathbf U$ and $\mathbf X_0$. 
The distribution of the singular values will also heavily impact the tightness of the bound: if all singular values are the same (i.e., $\mathbf W_q$, $\mathbf W_k$ are each orthogonal matrices), the bound will be tight. If there are a few very large singular values and many small ones, the bound is loose.
Therefore, the reachability condition in the theorem can be overly optimistic when used as a test for reachability, though it remains a sufficient condition for \textit{un}reachability.

Theorem~\ref{thm:attention-control} and Theorem~\ref{thm:attention-control-improved} focus exclusively on the self-attention mechanism and do not directly address the impact of activation functions and other non-linearities present in the full transformer architecture on the controllability of the final model outputs. In a typical transformer block, the output of the self-attention layer passes through a non-linear activation function, such as ReLU or GELU, before being combined with the residual connection and proceeding to the next layer. These non-linearities can affect the propagation of signals through the network and, consequently, the controllability of the end-to-end model. 

Analyzing controllability in the presence of multiple layers with interleaved non-linearities is an open problem. Investigating this challenge through the lens of, for instance, non-linear control theory has the potential to guide the design of transformer models with enhanced steerability and interpretability, which may advance the frontier of controllable and explainable AI systems. This direction has the potential to advance our understanding of the complex dynamics of large language models and develop principled approaches to controlling their behavior. However, significant research is still needed to realize this goal.

\section{Prompt Optimization Algorithms}
\label{sec:algs}

\paragraph{Greedy Back-Generation: } While testing all prompts in $\mathcal V^k$ is intractable for $k >1$, it takes only $|\mathcal V|$ forward passes of the network to compute the loss on $y$ induced by all possible \textit{single token} prompts $u \in \mathcal V$. Our Greedy Back Generation algorithm leverages this fact to generate prompts $u\in \mathcal V^k$ one token at a time, working backward sampling the $i$th greedy-optimal single token extension $u' = \arg\max_{u'} P_{LM}(y|u'+u+x)$ of the current prompt $u\in \mathcal V^{i-1}$. 

\begin{algorithm}[tb]
\caption{Greedy Token-Wise Prompt Generation}
\label{alg:greedy}
\begin{algorithmic}[1]
   \REQUIRE A causal LLM $P_{LM}$ with vocabulary $\mathcal V$, a set of base tokens $x \in \mathcal V^n$, a desired final token $y\in \mathcal V$, and a desired number of prompt tokens $k$.
   \ENSURE \textit{Magic words} $u^*$ of length $k$.
   \STATE Initialize $u^*$ to be empty.
   \FOR{$i=1$ {\bfseries to} $k$}
      \FORALL{$u' \in \mathcal V$}
         \STATE compute $P_{LM} (y | u'+u^*+x)$
      \ENDFOR
      \STATE Select the $u'$ that maximizes the probability of $y$ given $u' + u^* + x$. Prepend $u'$ to $u^*$
   \ENDFOR
   \STATE \textbf{return} $u^*$
\end{algorithmic}
\end{algorithm}
This method is optimal for $k=1$ prompt token $u^{*}\in \mathcal V$ and generally outperforms GCG for short prompts of length $k\leq 3$. Computing 1 additional prompt token takes roughly 1-4 minute when using an NVIDIA A100-80GB GPU with a 7 billion parameter model and 5-20 minutes on 2 NVIDIA A100-80GB GPUs with a 40 billion parameter model.

\paragraph{Greedy Coordinate Gradient (GCG): } The Greedy Coordinate Gradient algorithm, presented by \citep{zou2023universal} building off the work of \citep{shin2020autoprompt}, is the state-of-the-art method for optimizing prompts. Starting with a random prompt of length $k$, the algorithm generates a batch of alternative prompts. Each member of the batch swaps a random token in the current prompt with a promising alternate token. The value metric for a swap is given by a first order approximation of the change in loss $\mathcal L = \text{CELoss}(y, P_{LM}(y | u + x))$ with the embedding of each token in $u$. 

\begin{algorithm}[tb]
\caption{Greedy Coordinate Gradient}
\label{alg:gcg}
\begin{algorithmic}[1]
   \REQUIRE A causal LLM \( P_{LM} \) that accepts token strings from a vocabulary \( \mathcal X \), an embedding dictionary \( \mathbf{e} \), embeddings \( \mathbf{e}^{*}_i \) corresponding to each token \( i \) of \( u^* \), a set of base tokens \( x_{1:n} \), a desired number of prompt tokens \( k \), iterations \( T \), \( k_{sub} \), and batch size \( B \).
   \ENSURE \textit{Magic words} \( u^* \) of length \( k \).
   \STATE Initialize \( u^* \) to be random tokens from vocabulary.
   \FOR{\( iteration = 1 \) \textbf{to} \( T \)}
      \FOR{\( i = 1 \) \textbf{to} \( k \)}
         \STATE \( \mathcal X_i = \) Top-\( k_{sub} \) (\( \mathbf{e}^\top \nabla_{\mathbf{e}^{*}_i} P_{LM}(x_n | u^* + x_{1:n-1}) \)) \label{line:cand_calc}
      \ENDFOR
      \FOR{\( b = 1 \) \textbf{to} \( B \)} \label{line:for_batch}
         \STATE \( i = \) randint(\([1, \dots, k]\))
         \STATE \( j = \) randint(\([1, \dots, k_{sub}]\))
         \STATE \( \tilde{u}^{*}_b [i] = \mathcal X_i [j] \) 
      \ENDFOR
      \STATE \( u^{*} = \tilde{u}^{*}_{b^\ast} \), where \( b^\ast = \) argmax\(_b (P_{LM}(x_n | u^* + x_{1:n-1}))\)
   \ENDFOR
   \STATE \textbf{return} \( u^* \)
\end{algorithmic}
\end{algorithm}

This method outperforms all other methods we tested for prompts of length $k>3$. We use a batch size $B=768$, sampled from the top $k_{sub}=128$ token replacements at each index, and iterate for $T=34$ iterations. 
For each instance, this optimization took roughly 2 minutes for the 7 billion parameter models on a single A100-80GB GPU and 4-8 minutes for the 40 billion parameter model on 4 A100-80GB GPU.

\section{Supplementary Figures: Optimal Control Prompts}
\label{sec:sup_figs}

\subsection{``Ground Truth'' Controllability Results}
\label{sec:sup_main}
This subsection includes supplementary figures for the controllability of Llama-7b, Falcon-7b, and Falcon-40b ``ground truth'' target outputs from Wikitext. For each initial state sequence $\mathbf x_0$, the target output $y$ is the token immediately following $\mathbf x_0$ in Wikitext. 
We measured the $k$-$\epsilon$ controllability of each of the 7 billion parameter models with a dataset of 5000 state-output pairs while we used a dataset of 500 state-output pairs for Falcon-40b. 

Figure~\ref{fig:log_main} shows each model's log-spaced $k$-$\epsilon$ curves on the Wikitext dataset, revealing a log-linear relationship between maximum prompt length $k$ and the fraction of uncontrollable initial state-target output pairs $(\mathbf x_0, y)$. 
We visualize the relationship between prompt length and the prior cross-entropy loss of each LLM on predicting the target output $y$ given the state sequence $\mathbf x_0$ (i.e., $-\log P_{LM}(y | \mathbf x_0)$ in Figure~\ref{fig:base_loss_k_main} where we find it difficult to predict the required prompt length from the base loss. 

Finally, Figure~\ref{fig:main_freqs} shows a histogram of the tokens in the optimized prompts generated in the ground truth $k$-$\epsilon$ controllability experiments on Wikitext. 

\begin{figure}[ht]
    \centering
    \begin{minipage}[b]{0.48\textwidth}
        \includegraphics[width=\textwidth]{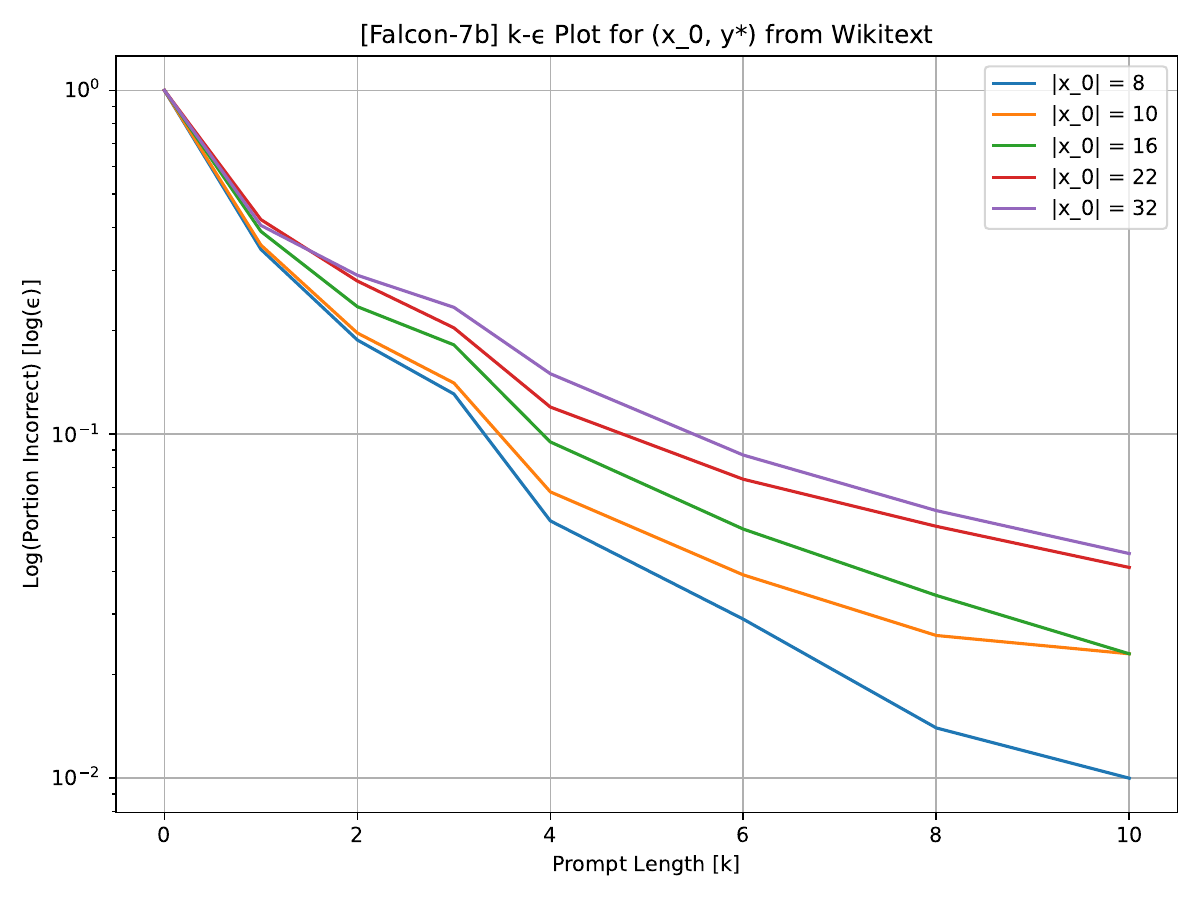}
    \end{minipage}
    \hfill
    \begin{minipage}[b]{0.48\textwidth}
        \includegraphics[width=\textwidth]{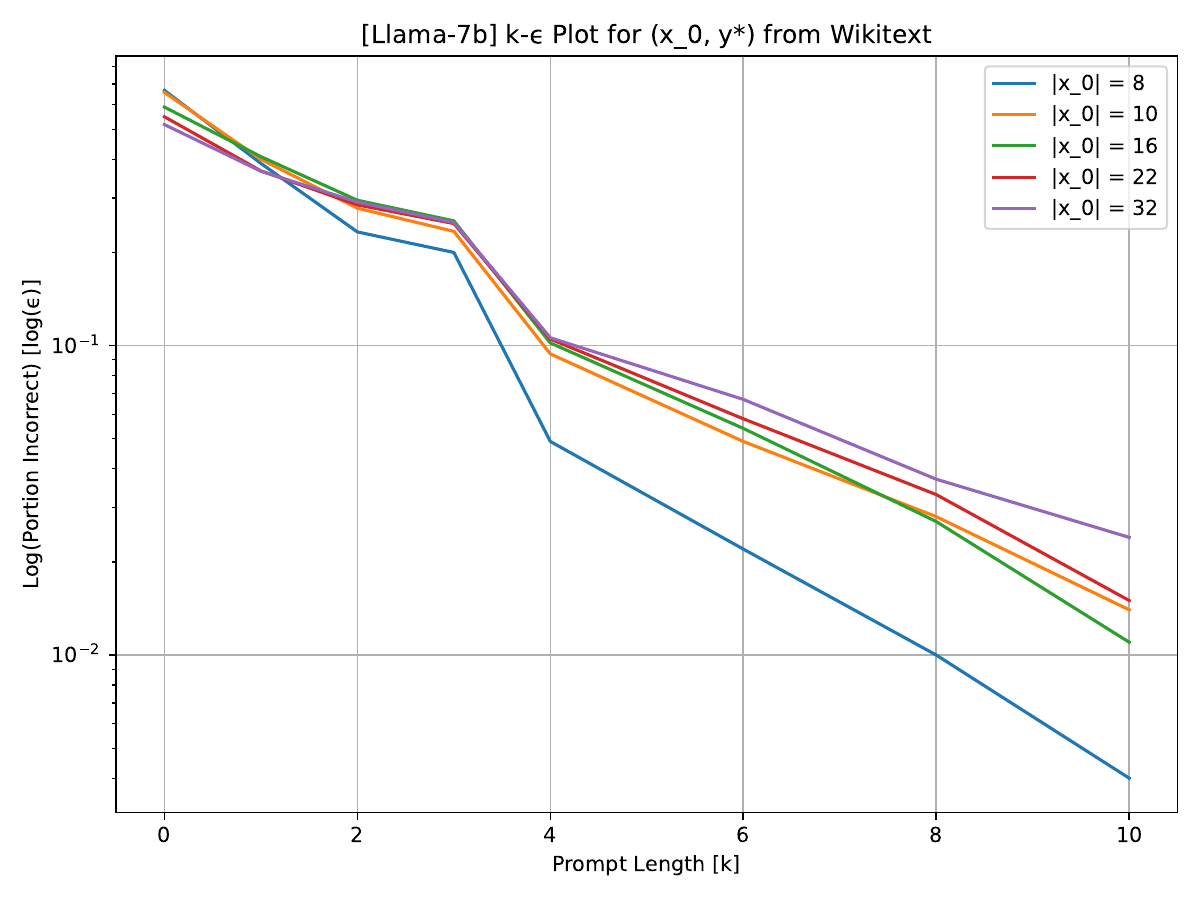}
    \end{minipage}
    \vspace{1em} %
    \begin{minipage}[b]{0.48\textwidth}
        \caption{
            Log spaced main results of $k$-$\log(\epsilon)$ controllability. Interestingly, the relationship between $k$ and $\log(\epsilon)$ appears roughly linear for each question length in the regime studied. \\
            \textbf{Top left}: $k$-$\log(\epsilon)$ values for Falcon-7b. With $k=10$ control tokens, 97.16\% of the target outputs were reachable. \\
            \textbf{Top right}: $k$-$\log(\epsilon)$ values for Llama-7b. With $k=10$ control tokens, 98.64\% of the target outputs were reachable. \\
            \textbf{Bottom right}: $k$-$\log(\epsilon)$ values for Falcon-40b. With $k=10$ control tokens, 97.00\% of the target outputs were reachable.
            \label{fig:log_main}
        }
    \end{minipage}
    \hfill
    \begin{minipage}[b]{0.48\textwidth}
        \includegraphics[width=\textwidth]{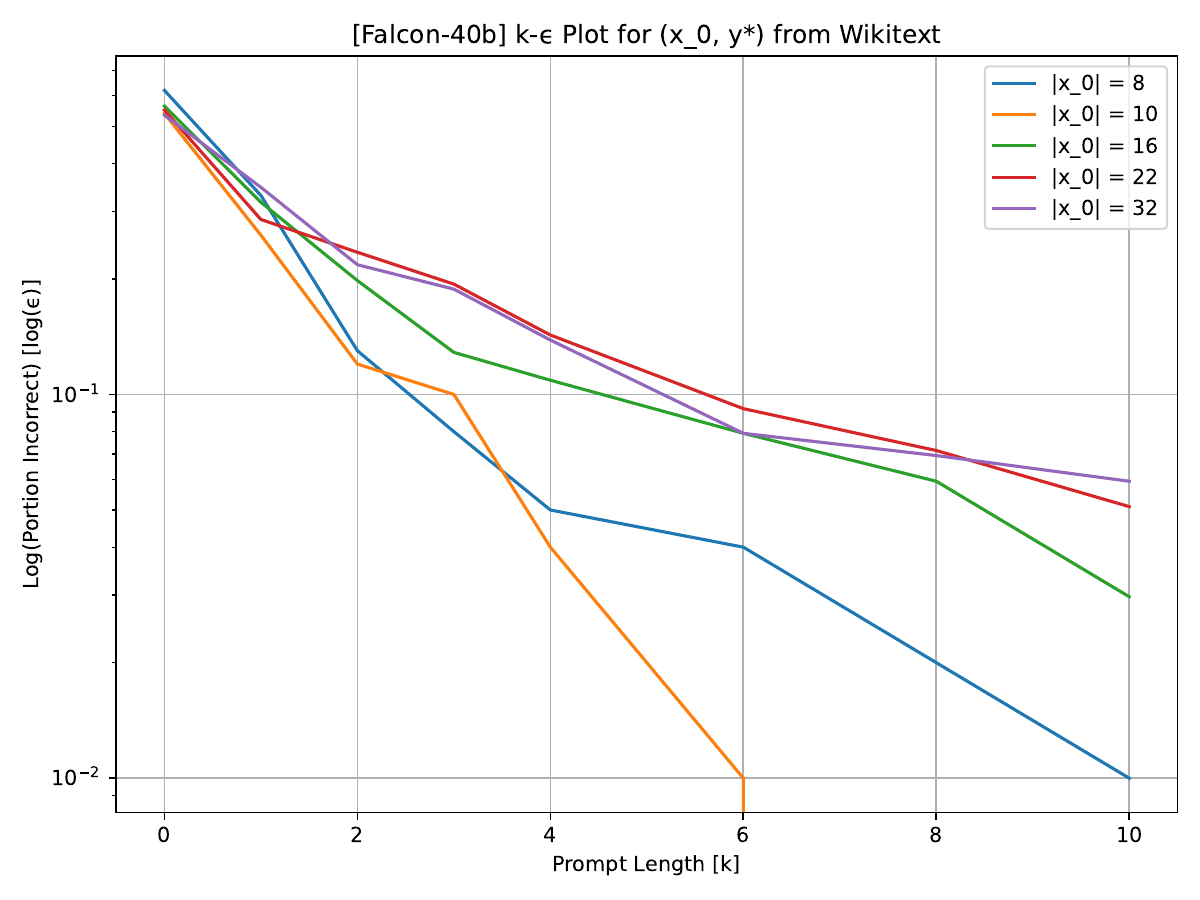}
    \end{minipage}
\end{figure}

\begin{figure}[ht]
    \centering
    \begin{minipage}[b]{0.48\textwidth}
        \includegraphics[width=\textwidth]{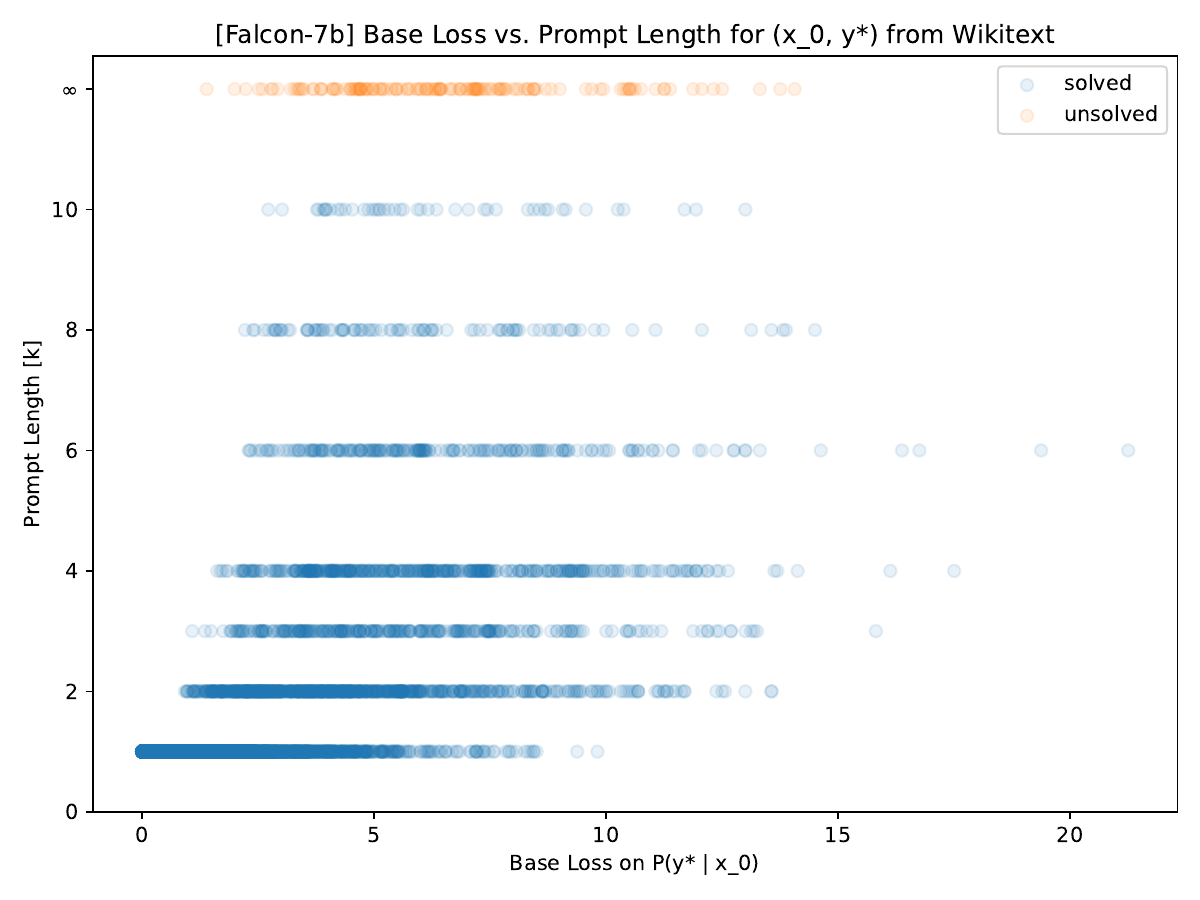}
    \end{minipage}
    \hfill
    \begin{minipage}[b]{0.48\textwidth}
        \includegraphics[width=\textwidth]{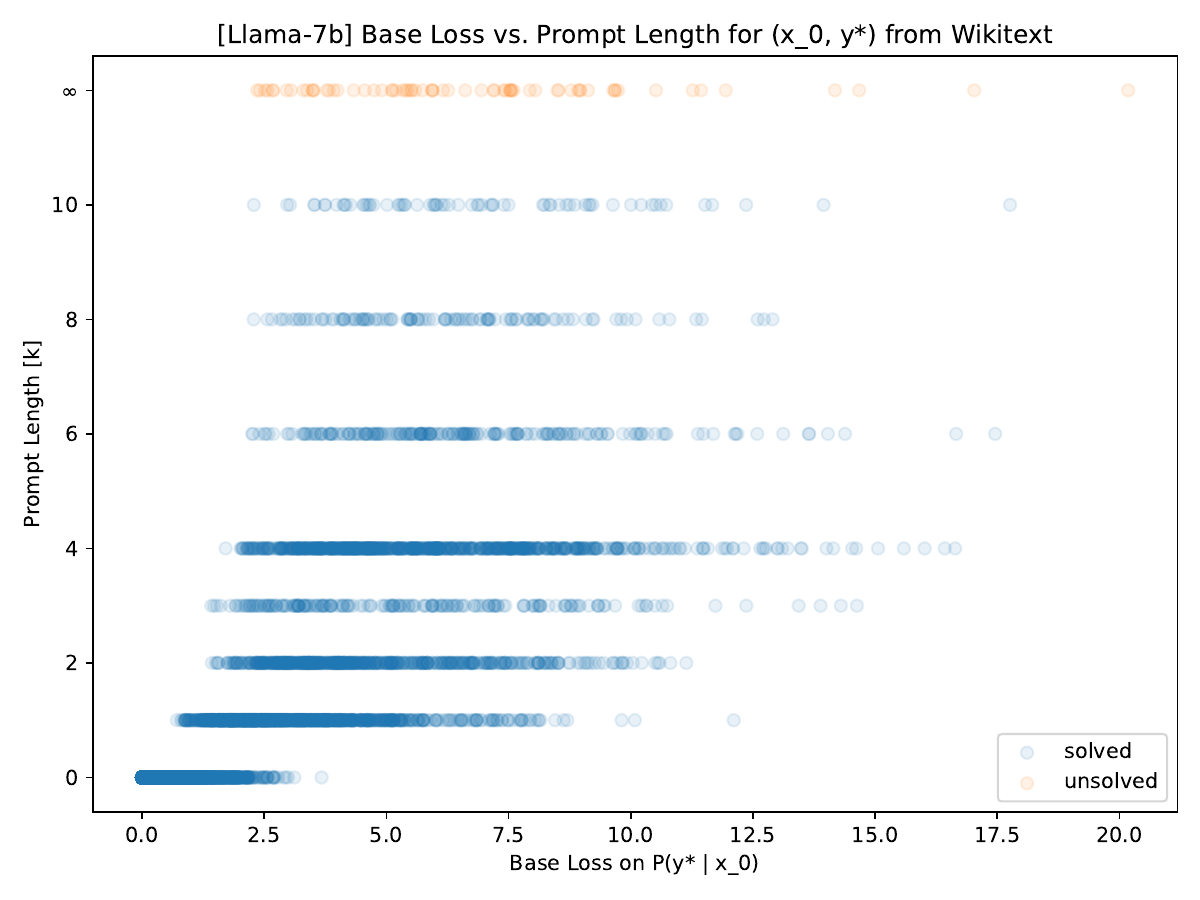}
    \end{minipage}
    
    \vspace{1em} %

    \begin{minipage}[b]{0.48\textwidth}
        \caption{
            Required prompt length $k$ versus base loss on the target output $\mathcal L = -\log P_{LM}(y | \mathbf x_0)$ on ``ground truth'' wikitext target outputs $y$ directly proceeding $\mathbf x_0$.
            \textbf{Top left}: Falcon-7b. 
            \textbf{Top right}: Llama-7b. 
            \textbf{Bottom right}: Falcon-40b.
            While there does appear to be an ``exclusion zone'' in the top left-hand corner where a high prompt length is never associated with a base loss below a given threshold, base loss appears to be a poor predictor of required prompt length.
            \label{fig:base_loss_k_main}
        }
    \end{minipage}
    \hfill
    \begin{minipage}[b]{0.48\textwidth}
        \includegraphics[width=\textwidth]{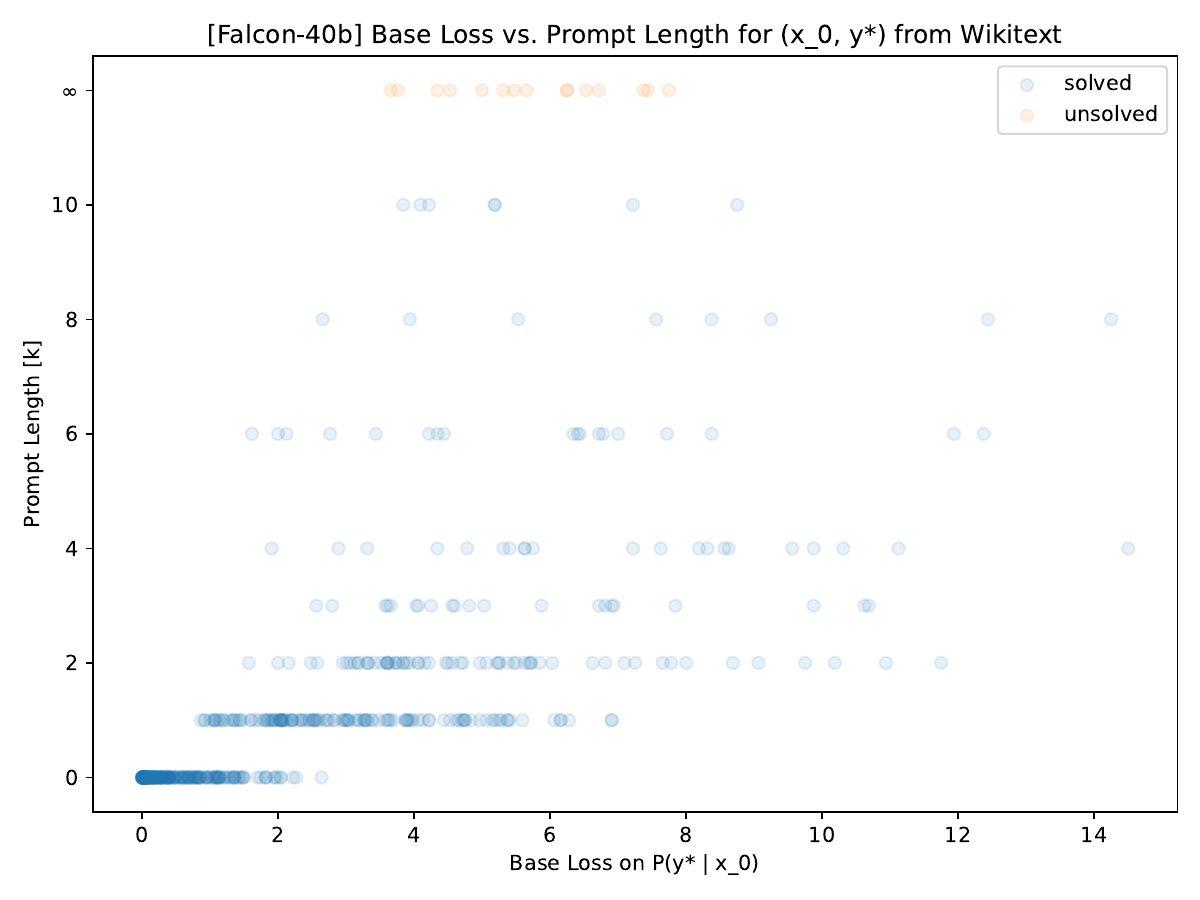}
    \end{minipage}
\end{figure}

\begin{figure}[htbp]
    \centering
    \subfigure[Falcon-7b]{
        \includegraphics[width=0.9\textwidth]{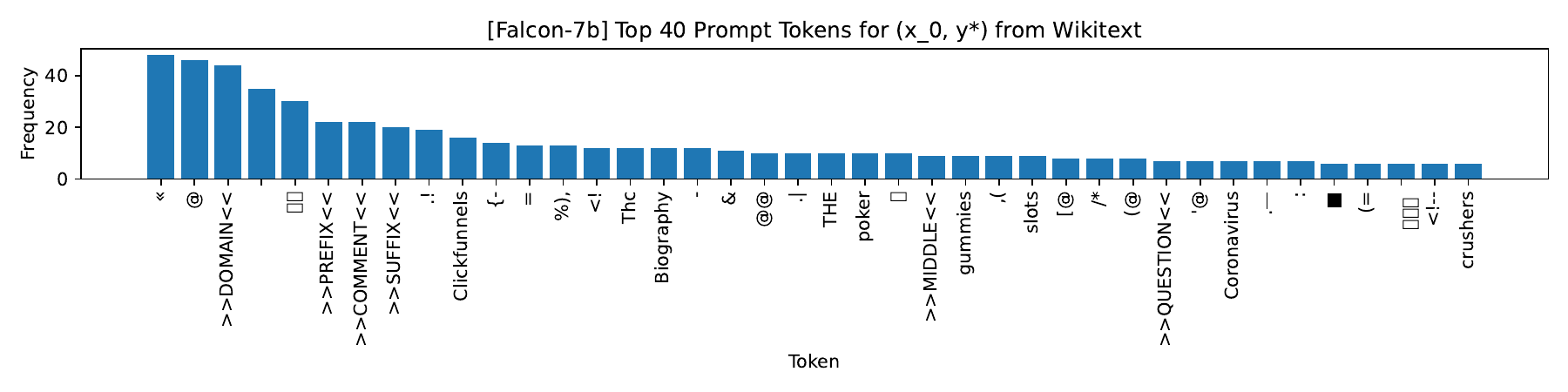}
    }
    \subfigure[Llama-7b]{
        \includegraphics[width=0.9\textwidth]{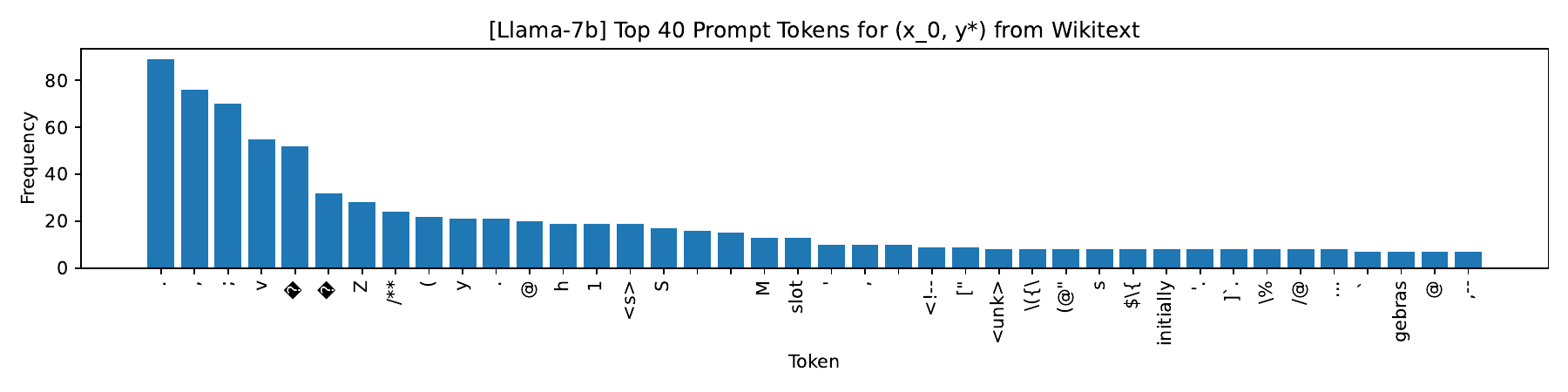}
    }
    \subfigure[Falcon-40b]{
        \includegraphics[width=0.9\textwidth]{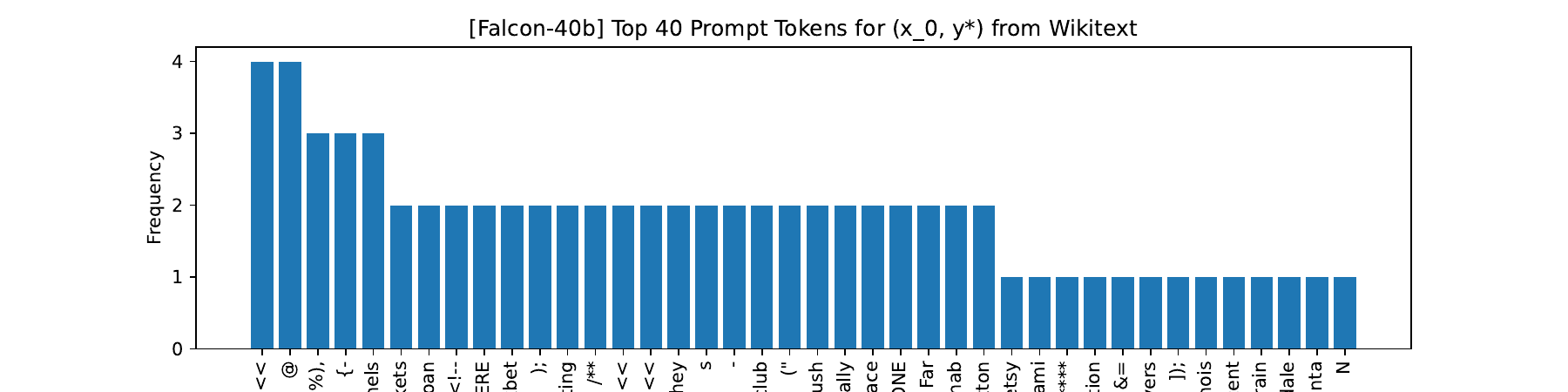}
    }
    \caption{Prompt token frequencies for Falcon-7b (top), Llama-7b (middle), and Falcon-40b (bottom) from Wikitext ground truth target token $k$-$\epsilon$ controllability experiments.}
    \label{fig:main_freqs}
\end{figure}

\subsection{Top-75 Wikitext Controllability Results} 
\label{sec:sup_75}
This subsection includes supplementary figures for the controllability of Llama-7b, Falcon-7b, and Falcon-40b on the Wikitext dataset where the target output token $y$ for a given initial state token sequence $\mathbf x_0$ is sampled uniformly from the top 75 highest-probability tokens as determined by the language model itself $P_{LM}(y | \mathbf x_0)$. 
Specifically, the dataset $\mathcal D$ consists of 25 unique initial state token sequences $\mathbf x_0$ sampled from Wikitext, each replicated 75 times for the top 75 most probable subsequent tokens $y \sim P(y | \mathbf x_0)$. 
This procedure yielded a dataset of 1875 initial state-target output pairs $(\mathbf x_0, y)$ for the 7 billion parameter models.
Due to the computational requirements for the 40 billion parameter model, the number of unique initial state token sequences was decreased to 10, resulting in a dataset of 750 initial state-target output pairs. 
The $k$-$\epsilon$ plots for each model are shown in Figure~\ref{fig:k_eps_75}. On average, across the 3 models, the top 75 outputs were reachable 86.865\% of the time with $k\leq 10$ prompt tokens. 
Similar log-linear trends were observed in the $k$-$\epsilon$ plot.
Figure~\ref{fig:base_loss_k_75} shows the relationship between base loss and required prompt length, revealing a more dramatic ``exclusion zone'' in the top left, similar to main ``ground truth'' results in Figure~\ref{fig:base_loss_k_main}.
Finally, Figure~\ref{fig:75_freqs} plots a histogram of the 40 most common tokens observed in the optimized control input prompts from the top-75 experiments.

\begin{figure}[ht]
    \centering
    \begin{minipage}[b]{0.48\textwidth}
        \includegraphics[width=\textwidth]{figs/shallow1_falcon7b_k_epsilon.pdf}
    \end{minipage}
    \hfill
    \begin{minipage}[b]{0.48\textwidth}
        \includegraphics[width=\textwidth]{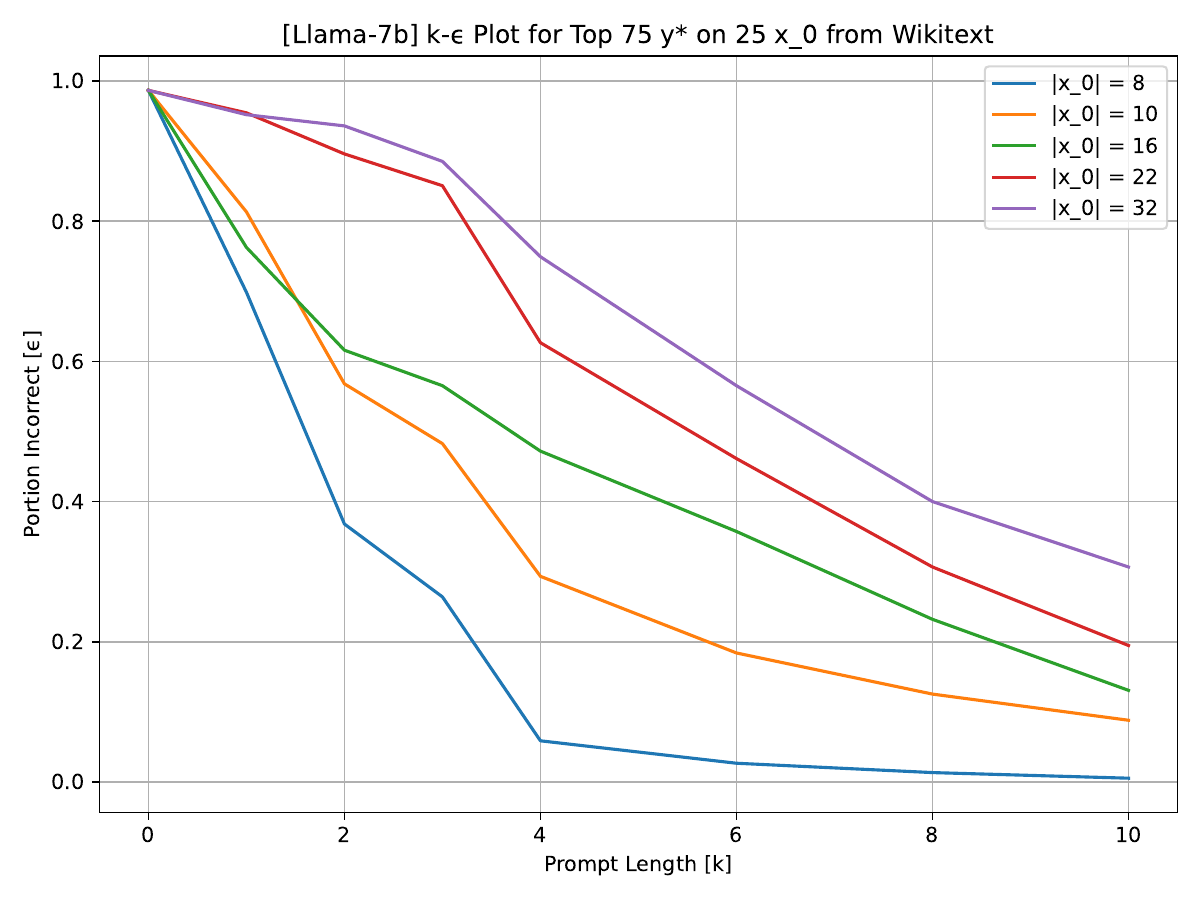}
    \end{minipage}
    \vspace{1em} %
    \begin{minipage}[b]{0.48\textwidth}
        \caption{
            $k$-$\epsilon$ controllability plots on the top 75 most likely output tokens.  \\
            \textbf{Top left}: $k$-$\epsilon$ values for Falcon-7b. With $k=10$ control tokens, 89.387\% of the top 75 output tokens were reachable. \\
            \textbf{Top right}: $k$-$\epsilon$ values for Llama-7b. With $k=10$ control tokens, 85.493\% of the top 75 output tokens were reachable. \\
            \textbf{Bottom right}: $k$-$\epsilon$ values for Falcon-40b. With $k=10$ control tokens, 85.714\% of the top 75 output tokens were reachable.
            \label{fig:k_eps_75}
        }
    \end{minipage}
    \hfill
    \begin{minipage}[b]{0.48\textwidth}
        \includegraphics[width=\textwidth]{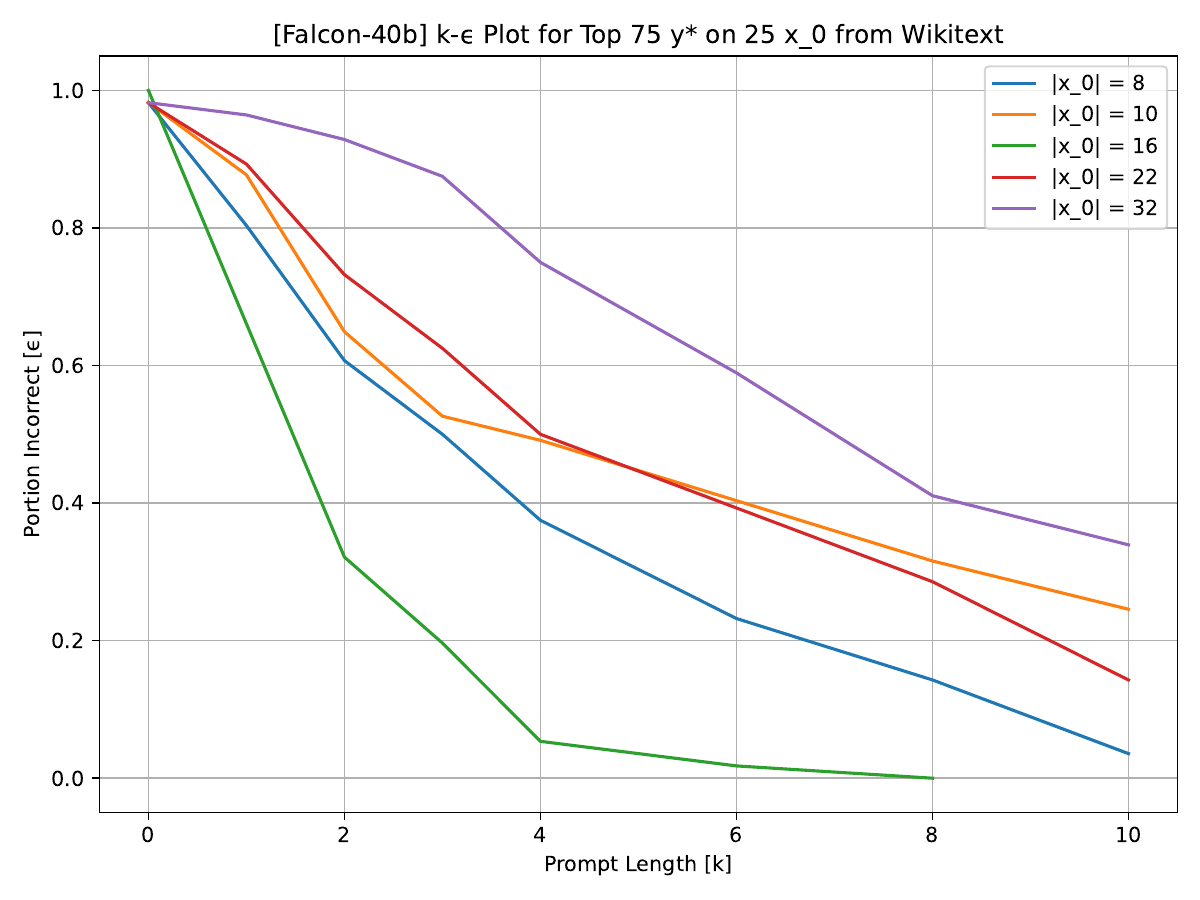}
    \end{minipage}
\end{figure}

\begin{figure}[ht]
    \centering
    \begin{minipage}[b]{0.48\textwidth}
        \includegraphics[width=\textwidth]{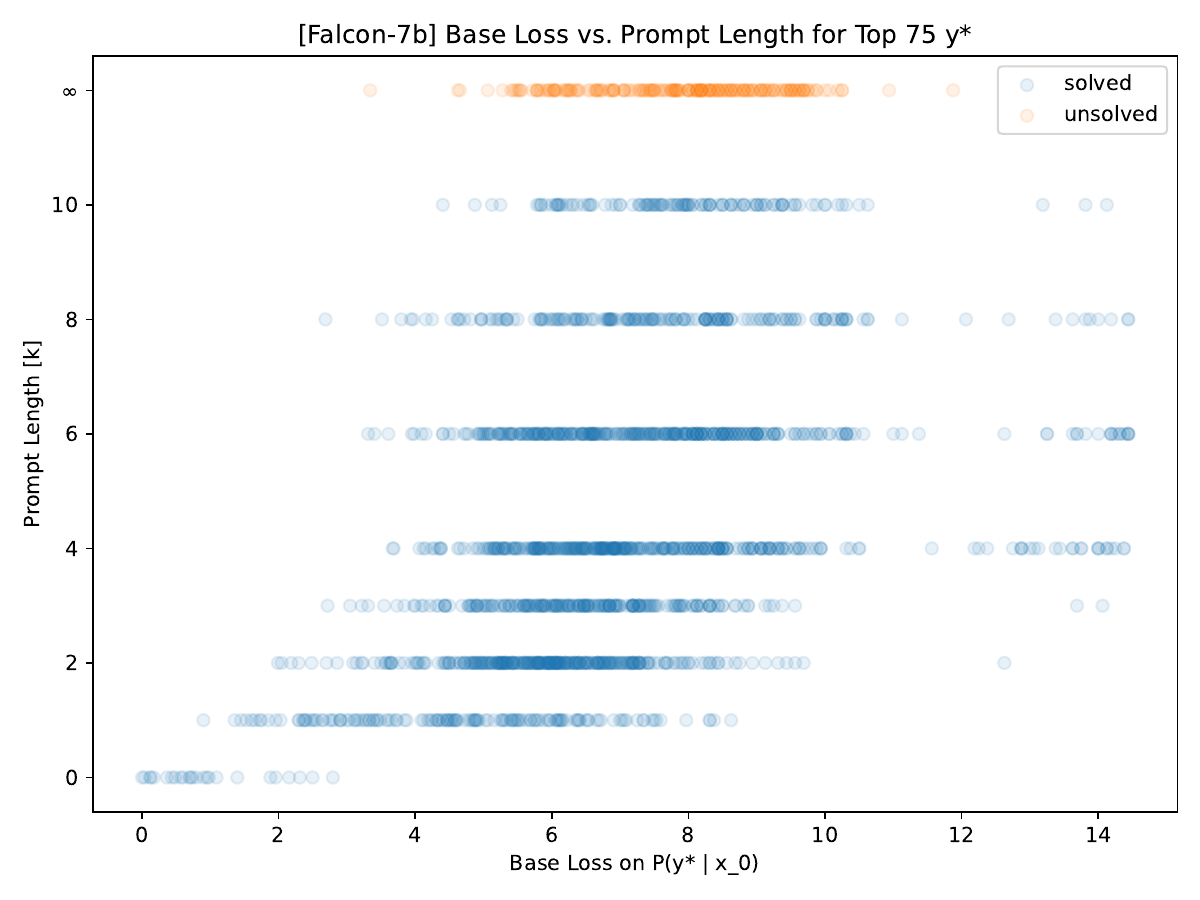}
    \end{minipage}
    \hfill
    \begin{minipage}[b]{0.48\textwidth}
        \includegraphics[width=\textwidth]{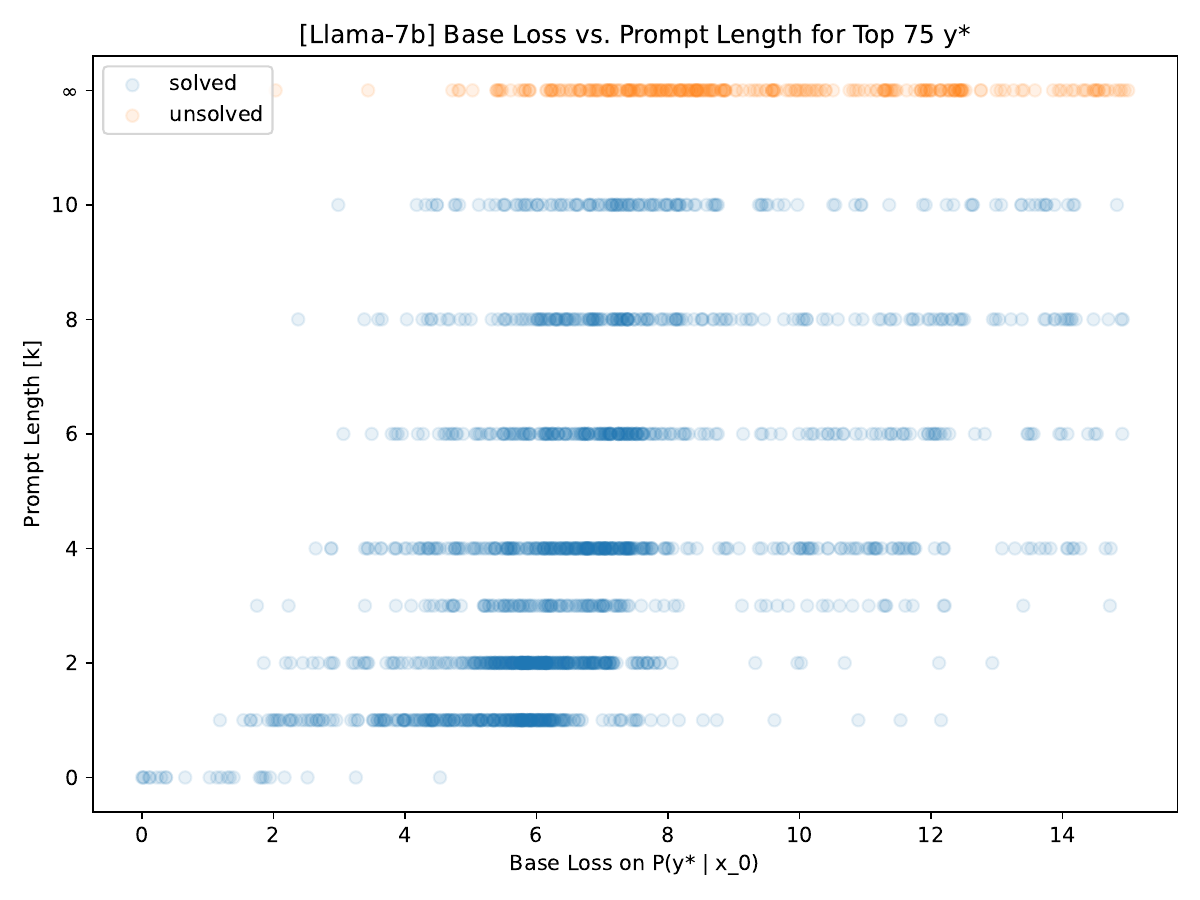}
    \end{minipage}
    
    \vspace{1em} %

    \begin{minipage}[b]{0.48\textwidth}
        \caption{
            Required prompt length $k$ versus base loss on the target output $\mathcal L = -\log P_{LM}(y | \mathbf x_0)$ on synthetic top-75 dataset.
            \textbf{Top left}: Falcon-7b. 
            \textbf{Top right}: Llama-7b. 
            \textbf{Bottom right}: Falcon-40b.
            While there does appear to be an ``exclusion zone'' in the top left-hand corner where a high prompt length is never associated with a base loss below a given threshold, base loss appears to be a poor predictor of required prompt length.
            \label{fig:base_loss_k_75}
        }
    \end{minipage}
    \hfill
    \begin{minipage}[b]{0.48\textwidth}
        \includegraphics[width=\textwidth]{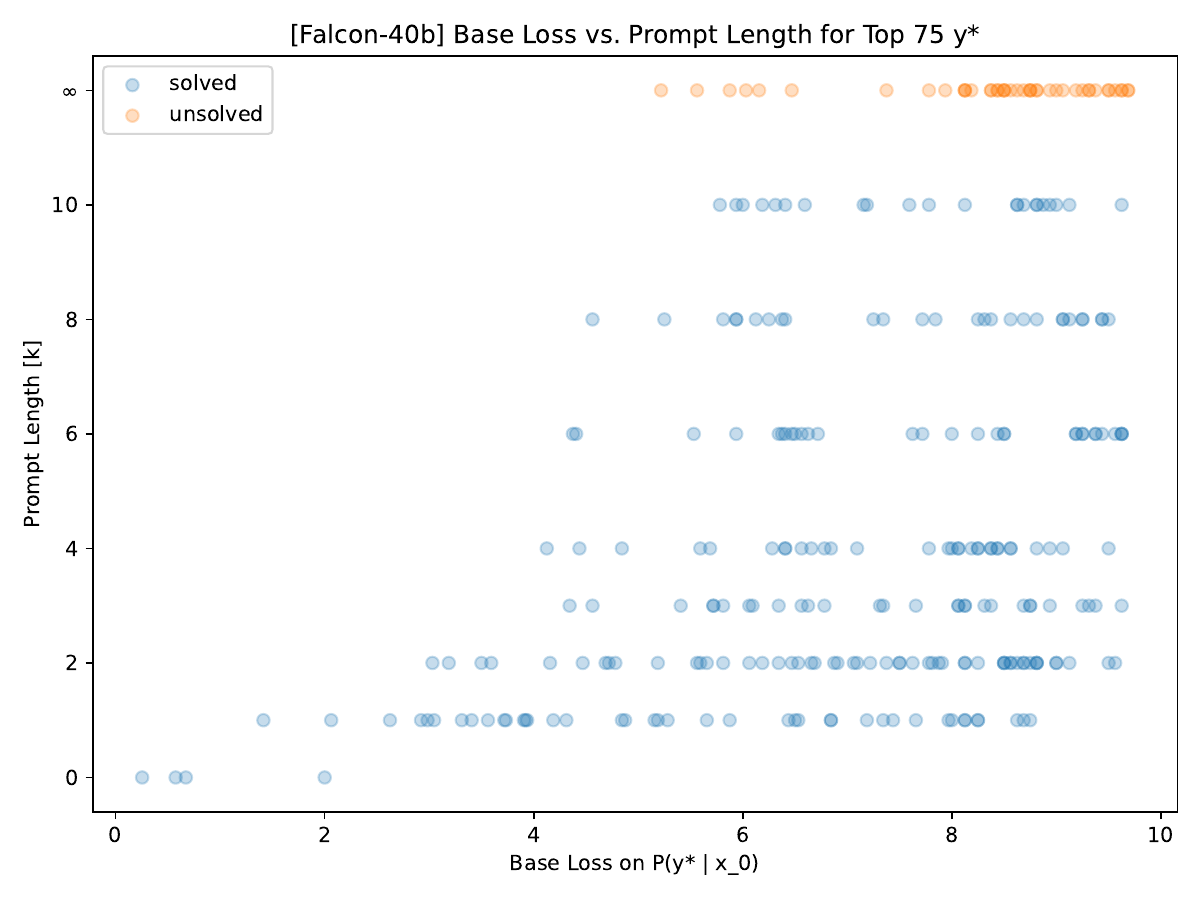}
    \end{minipage}
\end{figure}

\begin{figure}[htbp]
    \centering
    \subfigure[Falcon-7b]{
        \includegraphics[width=0.9\textwidth]{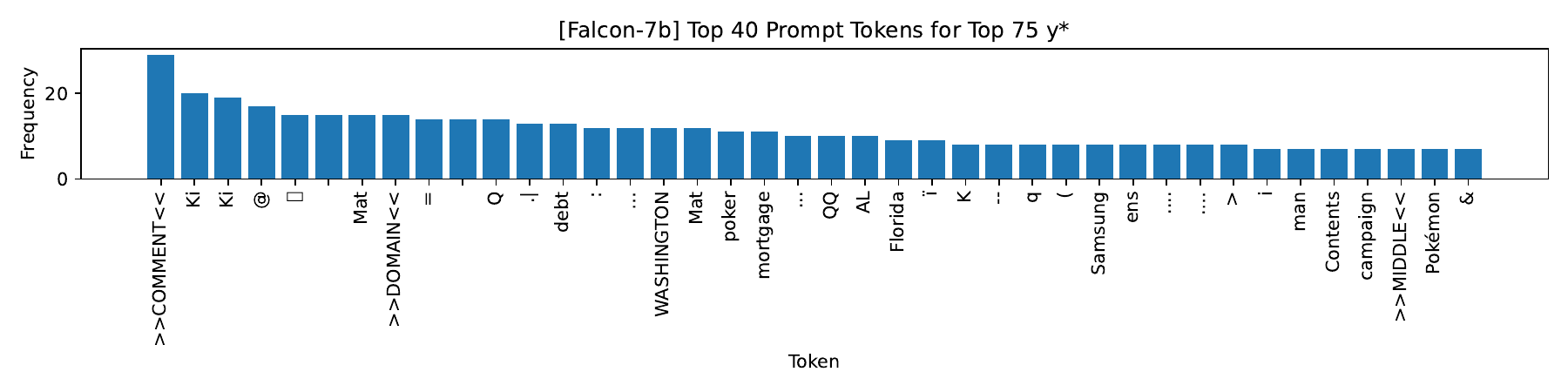}
    }
    \subfigure[Llama-7b]{
        \includegraphics[width=0.9\textwidth]{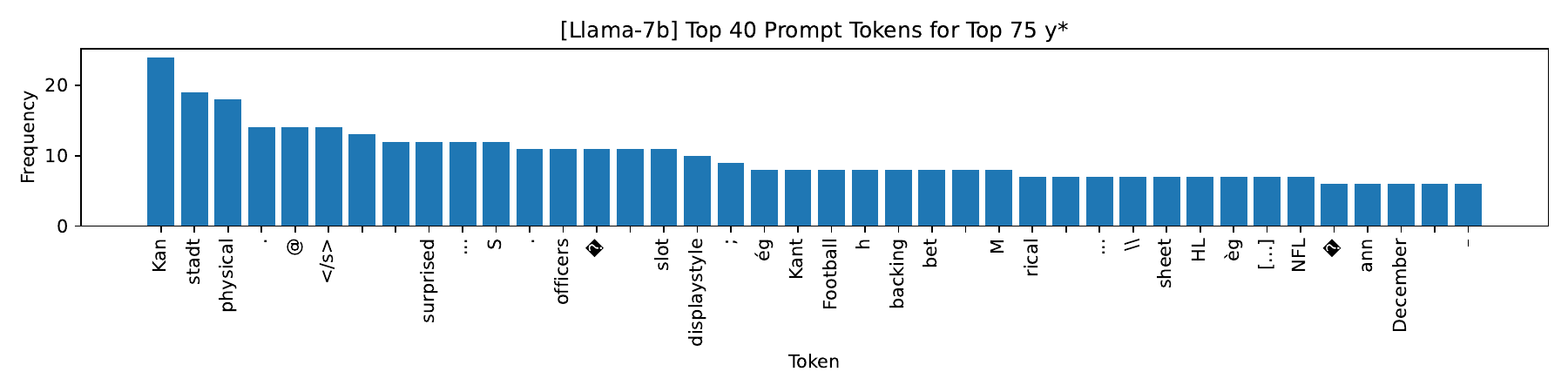}
    }
    \subfigure[Falcon-40b]{
        \includegraphics[width=0.9\textwidth]{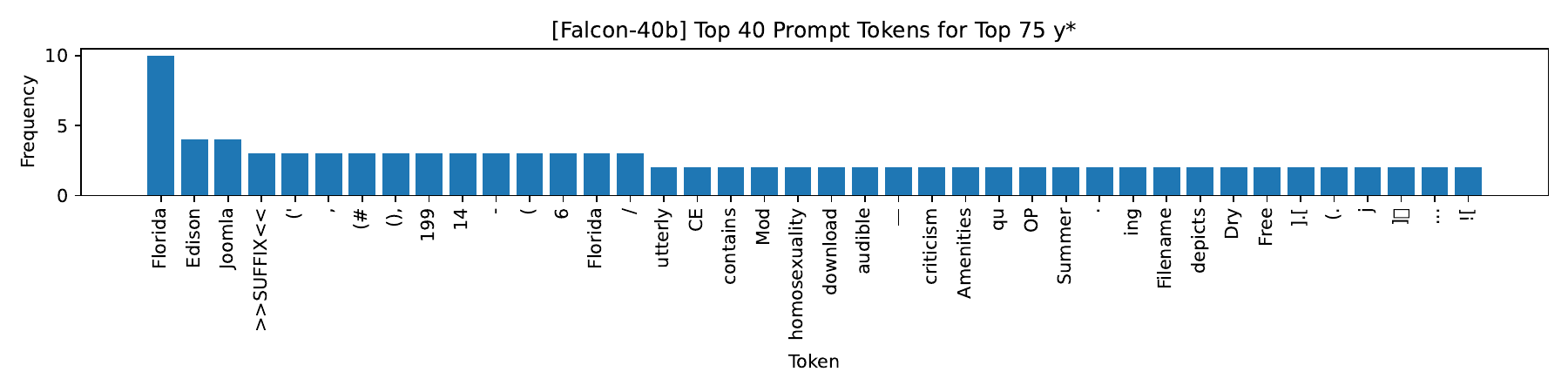}
    }
    \caption{Prompt token frequencies for Falcon-7b (top), Llama-7b (middle), and Falcon-40b (bottom) from Wikitext top-75 synthetic dataset $k$-$\epsilon$ controllability experiments.}
    \label{fig:75_freqs}
\end{figure}

\subsection{Uniformly Sampled Output Token Results}
\label{sec:sup_deep}

This section contains supplementary figures for $k$-$\epsilon$ controllability experiments on a synthetic dataset $\mathcal D = \{(\mathbf x_0, y)\}$ where $\mathbf x_0$ are sampled from the Wikitext dataset and $y$ is sampled uniformly from the vocabulary. 
The uniform target output dataset $\mathcal D$ consists of 616 state-output pairs. 
Due to computational constraints, $k$-$\epsilon$ controllability was only measured for Falcon-7b. 
Overall, only 46.42\% of the target outputs were reachable with $k=10$ prompt tokens. 
Figure~\ref{fig:deep_main} visualizes the $k$-$\epsilon$ results, the relationship between base loss and prompt length, and the most frequently observed tokens in the optimized control prompts. 
While the ``exclusion zone'' behavior (cf Figures~\ref{fig:base_loss_k_75}, \ref{fig:base_loss_k_main}) is observed in the base loss vs. prompt length subplot, base loss remains a poor predictor of required prompt length. 
Moreover, Figure~\ref{fig:deep_falcon_rank_k} reveals an even more uniform relationship between the initial rank of the target output token and the required prompt length. 

\begin{figure}[htbp]
    \centering
    \includegraphics[width=0.48\linewidth]{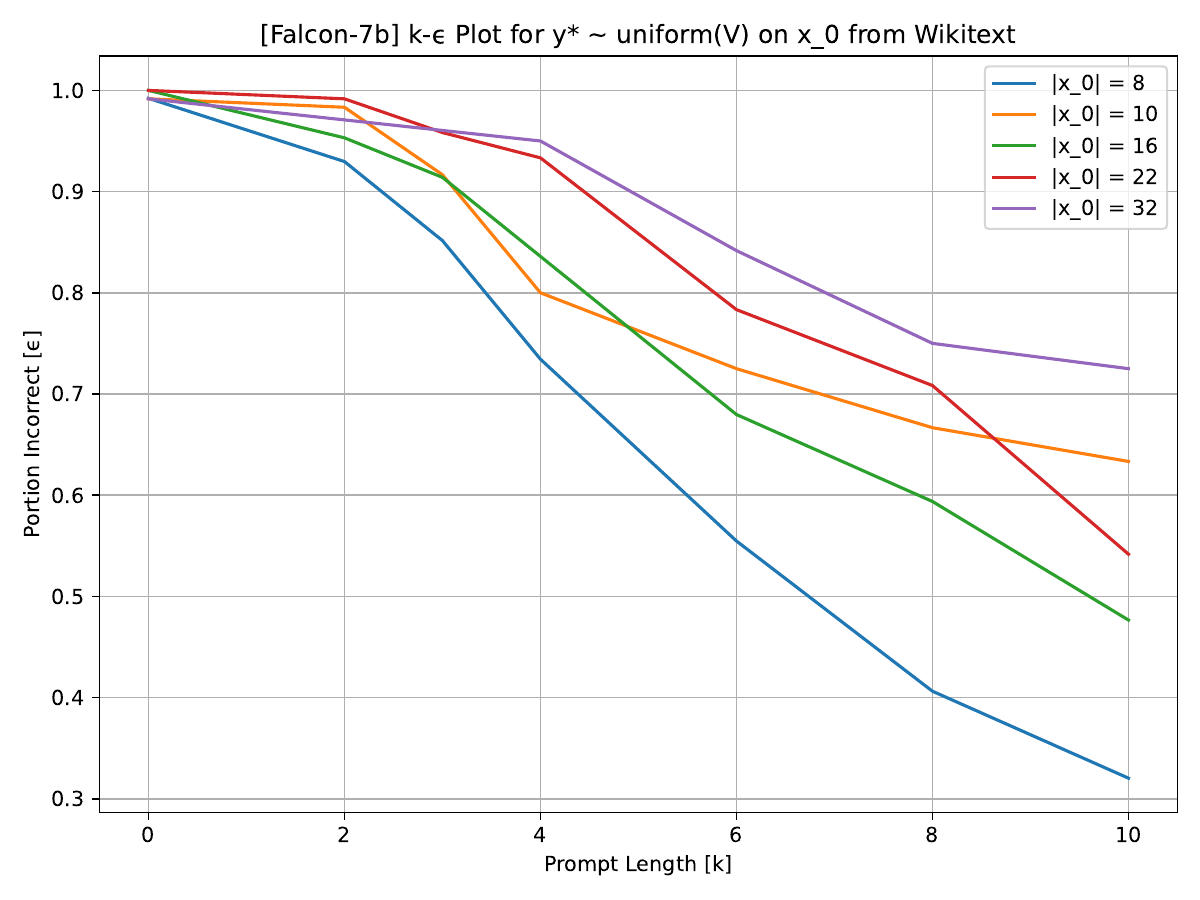}
    \hfill
    \includegraphics[width=0.48\linewidth]{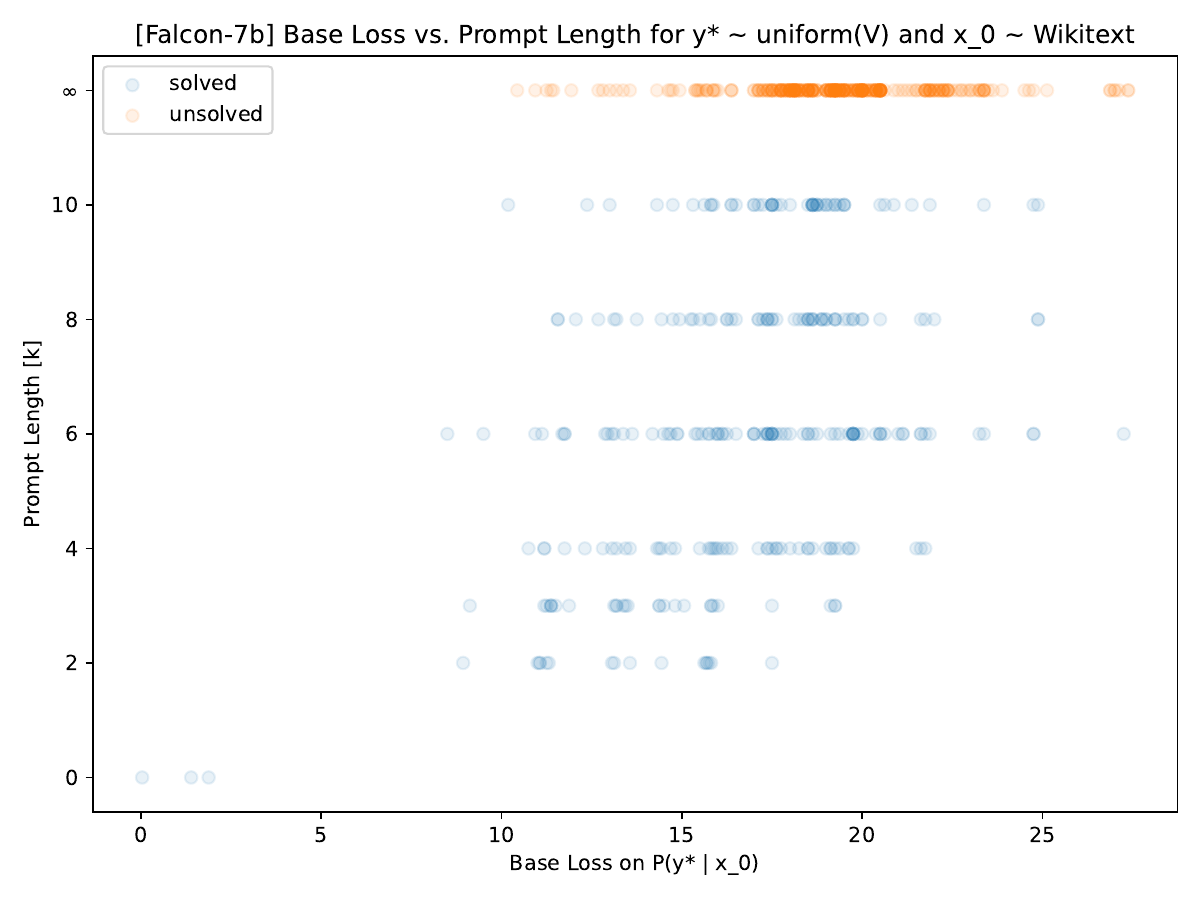}
    \includegraphics[width=\linewidth]{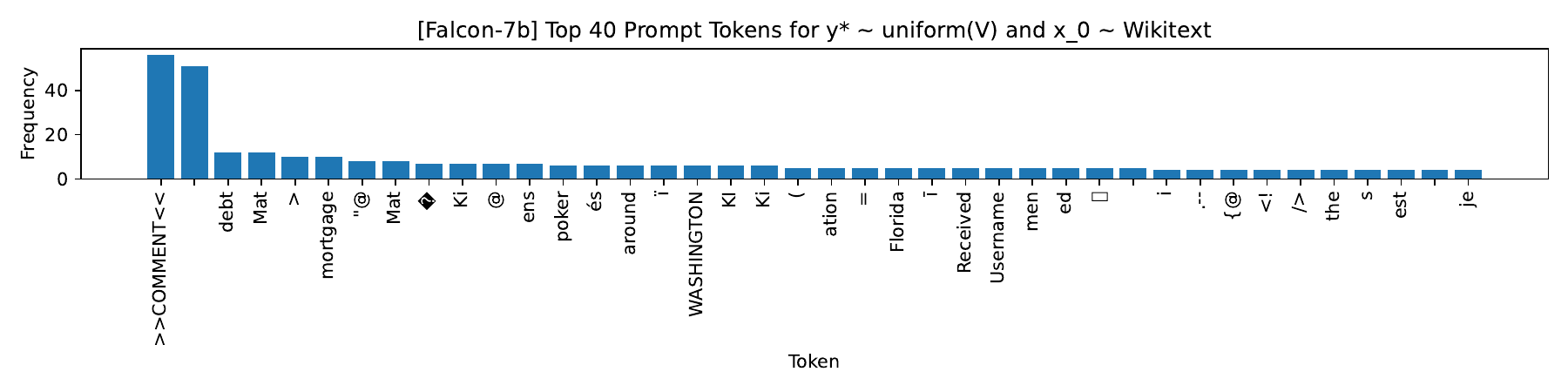}
    \caption{Supplementary figures on uniformly sampled target output controllability tests on Falcon-7b. \textbf{Top Left:} $k$-$\epsilon$ plot (46.42\% controllable at $k=10$). \textbf{Top Right:} Base loss versus required prompt length. \textbf{Bottom:} Histogram of top 40 most frequent tokens in optimized control prompts.}
    \label{fig:deep_main}
\end{figure}

\newpage
\section*{Glossary of Symbols}

\textbf{Self-Attention:}
\begin{itemize}
    \item $\Xi$ : The self-attention mechanism, a mapping from $\mathbb{R}^{N \times d_{in}}$ to $\mathbb{R}^{N \times d_{out}}$  
    \item $\mathbf{X}$ : The input matrix to self-attention, $\mathbf{X} \in \mathbb{R}^{N \times d_{in}}$
    \item $N$ : The number of input token representations
    \item $d_{in}$ : The dimensionality of each input token representation 
    \item $d_{out}$ : The dimensionality of each output token representation
    \item $\mathbf{W}_q, \mathbf{W}_{\rm key}, \mathbf{W}_v$ : The query, key, and value projection weight matrices
    \item $d_{\rm key}$ : The dimensionality of the key vectors
    \item $\mathbf{Q}, \mathbf{K}, \mathbf{V}$ : The query, key, and value matrices
    \item $\mathbf{D}$ : The diagonal matrix used for normalization
    \item $\mathbf{1}_{N \times 1}$ : An $N \times 1$ matrix of ones
\end{itemize}

\textbf{Input Partitioning:}  
\begin{itemize}
    \item $\mathbf{U}$ : The $k \times d_{in}$ submatrix of $\mathbf{X}$ corresponding to the control input
    \item $\mathbf{X}_0$ : The $m \times d_{in}$ submatrix of $\mathbf{X}$ corresponding to the imposed state
    \item $k$ : The number of control input tokens
    \item $m$ : The number of imposed state tokens
\end{itemize}

\textbf{Output Partitioning:}
\begin{itemize}
    \item $\mathbf{U}'$ : The $k \times d_{out}$ submatrix of the output corresponding to the control input 
    \item $\mathbf{Y}$ : The $m \times d_{out}$ submatrix of the output corresponding to the imposed state
    \item $\mathbf{Y}^*$ : The desired output, $\mathbf{Y}^* \in \mathbb{R}^{m \times d_{out}}$
    \item $\mathbf{Y}_u, \mathbf{Y}_x$ : The components of $\mathbf{Y}$ arising from $\mathbf{U}$ and $\mathbf{X}_0$ respectively 
    \item $\mathbf{Y}_{u,||}, \mathbf{Y}_{x,||}$ : The components of $\mathbf{Y}_u$ and $\mathbf{Y}_x$ parallel to $\mathbf{Y}^*$
    \item $\mathbf{Y}_{u,\perp}, \mathbf{Y}_{x,\perp}$ : The components of $\mathbf{Y}_u$ and $\mathbf{Y}_x$ orthogonal to $\mathbf{Y}^*$
    \item $\mathbf{Y}_{x,\perp}^{min}$ : The minimum value of $\mathbf{Y}_{x,\perp}$ over all control inputs that are uniformly bounded in norm by a fixed constant $M_u$ in the hypothesis of the theorem
\end{itemize}

\textbf{Reachability Conditions:}
\begin{itemize}
    \item $\beta_i(\mathbf{X}_0, k)$ : The upper bound on the norm of row $i$ of $\mathbf{Y}_{u,\perp}$
    \item $\gamma_i(\mathbf{X}_0, \boldsymbol{\theta})$ : A number that depends on $\mathbf{X}_0$ and $\boldsymbol{\theta} = (\mathbf W_q, \mathbf W_{\rm key}, \mathbf W_v)$
    \item $\alpha$ : An upper bound on the scaled key-query dot products
    \item $\sigma_q, \sigma_{\rm key}, \sigma_v$ : The maximum singular values of $\mathbf{W}_q$, $\mathbf{W}_{key}$, $\mathbf{W}_v$
    \item $M_u, M_x$ : The maximum norms of the control and imposed token embeddings
\end{itemize}

\end{document}